\documentclass{article}
\usepackage{iclr2027_conference,times}

\usepackage[utf8]{inputenc} % allow utf-8 input
\usepackage[T1]{fontenc}    % use 8-bit T1 fonts
\usepackage{hyperref}       % hyperlinks
\usepackage{url}            % simple URL typesetting
\usepackage{booktabs}       % professional-quality tables
\usepackage{amsfonts}       % blackboard math symbols
\usepackage{nicefrac}       % compact symbols for 1/2, etc.
\usepackage{microtype}      % microtypography
\usepackage{xcolor}         % colors
\usepackage{amsmath}
\usepackage{amsthm}
\usepackage{graphicx}
\usepackage{multirow}
\usepackage{adjustbox}
\usepackage{tikz}
\usepackage{xparse}
\usepackage{algorithm}
\usepackage{algorithmic}
\usepackage{xspace}
\usepackage{caption} 
\usepackage{array}

\usepackage{array}
\newcolumntype{C}[1]{>{\centering\arraybackslash}m{#1}}
\usepackage{makecell}

\newtheorem{theorem}{Theorem}

\newcommand{\static}{\texttt{Static FP-FM}\xspace}
\newcommand{\temporal}{\texttt{Temporal FP-FM}\xspace}
\newcommand{\dynamic}{\texttt{Dynamic FP-FM}\xspace}
\newcommand{\unconditional}{\texttt{Unconditional}\xspace}
\newcommand{\conditional}{\texttt{Conditional}\xspace}
\newcommand{\finetune}{\texttt{Finetune}\xspace}
\newcommand{\dg}{\texttt{Distribution-Guided}\xspace}
\newcommand{\cg}{\texttt{Classifier-Guided}\xspace}
\newcommand{\reference}{{\fontseries{b}\selectfont Context \xspace}}
\newcommand{\loo}{\texttt{Leave-One-Out}\xspace}

\newcommand{\algcomment}[1]{\textcolor{blue}{\small \it // #1}}
\NewDocumentCommand{\qualimg}{O{} O{0.08} O{south west} m}{%
\begin{tikzpicture}
        \node[inner sep=0pt] (img) {%
            \adjustimage{
                width=#2\textwidth,
                height=#2\textwidth,
                keepaspectratio=false,
                center,
                trim={0 0 0 0},
                clip
            }{qualitative_table/images/#4}%
        };
        \if\relax\detokenize{#1}\relax\else
            \node[
                anchor=#3,
                at=(img.#3),
                inner sep=1pt,
                fill=white,
                opacity=0.8,
                text opacity=1,
                font=\tiny\bfseries,
                text=black
            ] {#1};
        \fi
    \end{tikzpicture}%
}

\title{A Flow Matching Algorithm for Many-Shot Adaptation to Unseen Distributions}
\date{}

\author{%
Tyler Ingebrand\thanks{Equal contribution} \\
University of Texas at Austin \\
\texttt{\fontsize{7.7pt}{10pt}\selectfont tyleringebrand@utexas.edu} \\
\And
Ruihan Zhao\footnotemark[1]     \\
University of Texas at Austin \\
\texttt{\small ruihan.zhao@utexas.edu} \\
\And 
Kushagra Gupta \\
University of Texas at Austin \\
\texttt{\small kushagrag@utexas.edu} \\
\And
David Fridovich-Keil \\
University of Texas at Austin \\
\texttt{\small dfk@utexas.edu} \\
\And
Sandeep P. Chinchali \\
University of Texas at Austin \\
\texttt{\small sandeepc@utexas.edu} \\
\And
Ufuk Topcu \\
University of Texas at Austin \\
\texttt{\small utopcu@utexas.edu} \\
}

\iclrfinalcopy % Uncomment for camera-ready version, but NOT for submission.
\begin{document}
\maketitle
% Function space adaptation for many-shot flow matching TODO

\begin{abstract}

While generative modeling has achieved remarkable success on tasks like natural language-conditioned image generation, enabling model adaptation from  example data points remains a relatively underexplored and challenging problem. To this end, we propose Function Projection for Flow Matching (FP-FM), an algorithm that directly conditions generation on samples from the target distribution. FP-FM learns basis functions to span the velocity fields corresponding to a set of training distributions, and adapts to new distributions by computing a simple least-squares projection onto this basis. This enables efficient generation of samples from diverse target distributions without additional training at inference time. We further introduce multiple variants of FP-FM that provide a trade-off in expressivity and compute by enriching the coefficient calculation, e.g., by making the coefficients dependent on time. FP-FM achieves competitive precision and recall relative to baselines across synthetic and image-based datasets, with especially strong gains on unseen distributions, while requiring significantly lower compute.

\end{abstract}

\section{Introduction}

% Current state: generative modeling is wildly successful, but the typical approach is to condition on natural language. Not everything is representable that way. 
Generative modeling techniques such as diffusion and flow matching have demonstrated the ability to create high-fidelity, synthetic data modeled after a target distribution \citep{ddpm, fm,  stablediffusion}.
Generative modeling is useful in domains where the generated object is the goal, such as image editing or video generation, as well as in domains where the generated object is used for downstream tasks, i.e., acting as a world model for a reinforcement learning agent \citep{worldmodel, genie}. 
For many of these domains, the target distribution changes depending on the user's needs, and there may be limited samples from this target distribution.
Furthermore, it is often useful if the synthetic samples are generated with minimal latency.
In other words, we want a model capable of efficiently generating \textit{many} distributions based on user specifications.

% Other approaches for conditional distributions include ...
The most common approach for conditional generation is to include a conditioning variable as an additional input to the model. 
Most often, this conditioning variable takes the form of natural language (e.g., ``an image of a cat'') \citep{dalle, imagen}, but it may take on other forms such as a one-hot encoding of a class \citep{classifierfree}. 
However, conditioning variables are not always the best approach. 
For example, suppose we want to generate images of a specific person. A natural language description of that person's appearance is likely ineffective, while providing images of this person as examples is more intuitive. 
Conditioning variables are not sufficient for all tasks and we must consider algorithms which take samples from the target distribution as input. 

% Picture worth a thousand words; Let's instead condition on images from the target distribution
We introduce Function Projection for Flow Matching (FP-FM), an algorithm for conditional generation where the model is conditioned on samples from a target distribution (See Figure \ref{fig:concept}).  
FP-FM learns the $k$ most important basis functions in the function space of velocity fields induced by training distributions under flow matching dynamics.
By linearly combining these basis functions depending on the provided data, FP-FM is able to approximate the velocity field of a new distribution via a series of least-squares calculations. 
A concern for generative modeling is the fidelity-overfitting tradeoff, where highly expressive models may only reproduce the provided samples, and insufficiently expressive models generate infeasible samples. 
To this end, we present three FP-FM variants spanning this spectrum, allowing practitioners to choose the level of model expressivity based on their priorities.

\begin{figure}[t]
    \centering
    \includegraphics[width=1.0\linewidth,  trim=70pt 150pt 33pt 12pt,
    clip]{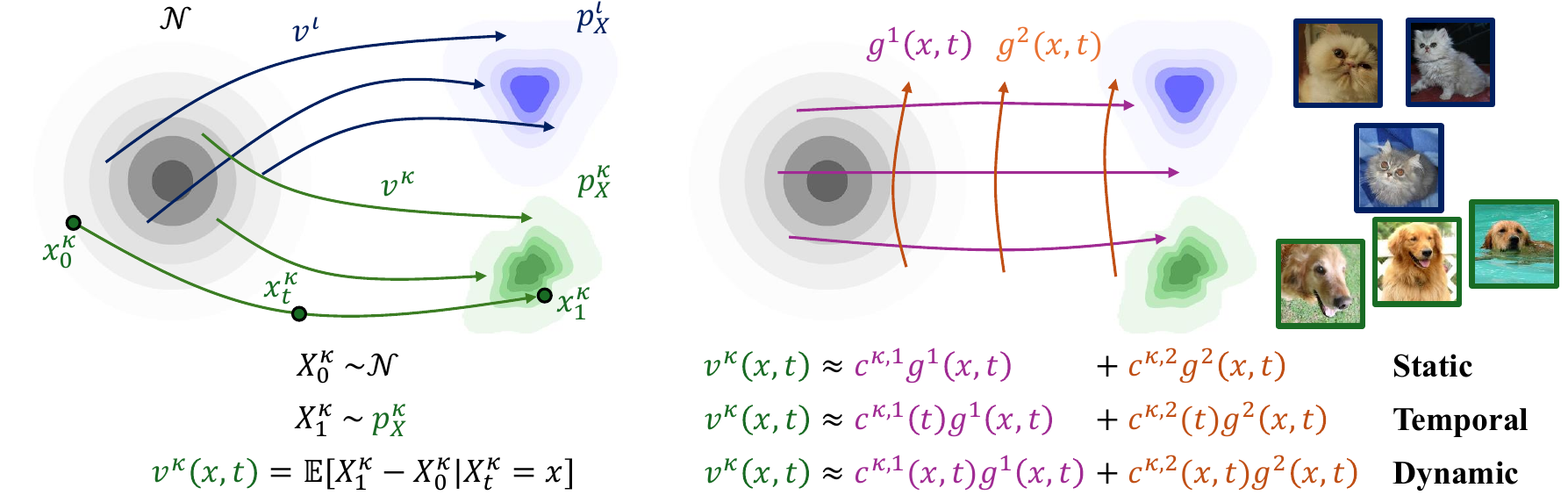}
    \caption{\textbf{Conceptual diagram.} (\textit{Left}) Illustration of two distributions, $p_X^\iota$ and $p_X^\kappa$, together with their associated velocity fields $v^\iota$ and $v^\kappa$. Probability densities are depicted as shaded regions, while velocity fields are indicated by arrows. Both stochastic processes share a common initial distribution, a Normal distribution shown in black.
    (\textit{Right}) FP-FM learns a set of basis functions that span the space of velocity fields, allowing a target field $v^\kappa$ to be represented as a linear combination of these bases. The three FP-FM variants, static, temporal, and dynamic, differ in whether the coefficients are fixed, time-dependent, or state and time-dependent, respectively.
    }
    \label{fig:concept}
\end{figure}

FP-FM is distinguished from prior work based on a few characteristics.
Unlike conditional models, FP-FM does not require a variable describing the target distributions; it only requires samples. 
On the other hand, FP-FM may match unseen target distributions without any training. 
It is computationally more efficient than finetuning, which requires expensive gradient steps to model a new distribution.

% Our results show...
We compare FP-FM against numerous baselines, including unconditional and conditional flow matching, classifier- and distribution-guided methods, and standard finetuning. Across a range of datasets and evaluation settings, we find that FP-FM achieves strong performance on both seen and unseen distributions. In particular, FP-FM consistently improves precision on out-of-distribution targets while maintaining competitive recall, highlighting its ability to capture the desired distribution without over-approximating the data manifold. Moreover, the different variants of FP-FM expose a clear tradeoff between expressivity and computational cost, allowing practitioners to select the most appropriate model for their setting. These results demonstrate that explicitly conditioning on samples using a function space perspective is a practical and effective alternative to existing approaches.

\section{Background}

We first clarify notation. We use $\iota \in \mathcal{I}$ to index over an arbitrary set and $i$ to index over a discrete set. Superscripts denote indices, e.g., $c^\iota$ or $c^i$. When multiple indices are required, we write $c^{\iota,i}$.
Subscripts denote time, such as a stochastic process $X_t$, or distributions, such as $p_X$.

\subsection{Flow Matching} \label{sec: fm}
% math
We build upon the flow matching framework \citep{fm}. 
Consider a data space $\mathcal{X} \subset \mathbb{R}^n$ and a probability density $p_X: \mathcal{X} \to \mathbb{R}_{> 0}$ with $\int_\mathcal{X}p_X(x) dx = 1$.
We are given $m$ samples $x^1, \dots, x^m \overset{\text{i.i.d.}}{\sim} p_X$.
The objective is to learn a generative model that produces samples $\hat{x}$ whose distribution approximately matches $p_X$. 
Flow matching defines a stochastic process, $X_t$ for $t\in[0,1]$, such that simulating this stochastic process approximately yields samples from $p_X$.  
To do so, choose a noise distribution that is easy to sample from, typically the standard normal distribution. 
Then  define $X_0 \sim \mathcal{N}(0,I)$, $X_1 \sim p_X$, and $X_t := (1-t)X_0 + tX_1$. 
% This path linearly interpolates from the noise distribution to the target distribution, but unfortunately simulating a specific realization of $X_t$ requires us to sample $X_1$, which we cannot do. 
To approximate this process, train a model $v$ to fit the expected velocity field, $v(x,t) = \mathbb{E}[X_1 - X_0|X_t=x]$, and define the approximated stochastic process as $X_0^{'} \sim \mathcal{N}(0, I), X_t^{'}:=X_0^{'}+\int_0^t v(X_\tau^{'}, \tau) d\tau$.
% Interestingly, these two processes are equal in distribution, $X_t^{'} \overset{d}{=}X_t$, meaning that marginal distributions at any $t$ are equal even though the paths taken by their respective samples are distinct. 
Then $X_t^{'} \overset{d}{=}X_t$ and it is feasible to simulate $X_t^{'}$, which generates samples from $p_X$. 

% However, simulating $X_t^{'}$ does not require us to sample $X_1$, and it is therefore feasible for sampling  \citep{fm}.\footnote{Since $X_t^{'} \overset{d}{=}X_t$, it is common to abuse notation by referring to both processes as $X_t$ even though samples of these processes follow very different paths.}
% Thus, to approximately sample from $p_X$, we simply sample from the noise distribution and then integrate the ODE defined by the trained model $v$.

% To train the velocity field, we first sample $X_0 \sim \mathcal{N}(0,I)$, $X_1 \sim \text{Unif}(\{x^i\}_{i=1}^m)$, and $t \sim \text{Unif}([0,1])$. Then we compute $X_t$ and the mean-squared error of $v_\theta(X_t,t)$ relative to $X_1 - X_0$. Note that the minimum mean squared error predictor for $X_1 - X_0 | X_t=x$ is exactly $\mathbb{E}[X_1 - X_0|X_t=x]$, motivating the use of mean-squared error as a loss function.
% While the velocity field itself may be a complicated object to learn, this formulation neatly fits into a standard supervised learning setting and thus many of the standard tools and intuition apply. 
% Furthermore, mean squared error has deep connections to inner products and Hilbert spaces, a connection we will exploit in the following sections. 

\subsection{Function Encoders} \label{sec: fe}

In this section, we introduce the function encoder (FE) algorithm \citep{fe}. 
% Let us first begin by defining key terms. 
 Consider a Hilbert space $\mathcal{H}=\{f: \mathcal{X} \to \mathbb{R}^n\}$ with an inner product $\mathcal{\langle \cdot, \cdot \rangle_\mathcal{H}}$.  
As an example, an inner product under an input data distribution $p_x$ is $\langle f,g\rangle_{p_X} := \frac{1}{n}\mathbb{E}_{X}[f(X)^\top g(X)]$.
In particular, this inner product has connections to mean squared error:
\begin{equation}
\begin{aligned}
% \|f - \hat{f}\|^2_{p_X} 
% &= \langle f - \hat{f}, f - \hat{f} \rangle_{p_X} \\
% &= \frac{1}{n}\mathbb{E}_{X }\big[(f(X) - \hat{f}(X))^\top (f(X) - \hat{f}(X))\big] \\
% &= \mathbb{E}_{X}\big[\frac{1}{n}\sum_{i=1}^n(f^i(X) - \hat{f}^i(X))^2\big]
\|f - \hat{f}\|^2_{p_X} = \mathbb{E}_{X}\big[\frac{1}{n}\sum_{i=1}^n(f^i(X) - \hat{f}^i(X))^2\big].
\end{aligned}
\end{equation}
Thus, mean squared error is equivalent to squared distance in this function space. 
The function encoder \citep{fe} is an algorithm to learn a set of basis functions  $\{g^i: \mathcal{X} \to \mathbb{R}^n \}_{i=1}^k$ parametrized as neural networks to span a function space, where any function $f \in \mathcal{H}$ is represented as a linear combination of the basis, $f = \sum_{i=1}^k c^ig^i$. Computing the coefficients  $c \in \mathbb{R}^k$ requires a dataset of input-output pairs, $\{(x^j, y^j)\}_{j=1}^m$. These input-output pairs are used to approximate the inner product, and the coefficients are calculated via regularized least squares:
\begin{equation} \label{eqn: least squares}
c = \left (\begin{bmatrix}
\langle g^1, g^1 \rangle_\mathcal{H} & \hdots & \langle g^1, g^k \rangle_\mathcal{H} \\
\vdots & \ddots & \vdots \\
\langle g^k, g^1 \rangle_\mathcal{H} & \hdots & \langle g^k, g^k \rangle_\mathcal{H} \\
\end{bmatrix} 
 + \lambda I\right) ^{-1}
\begin{bmatrix}
\langle f, g^1 \rangle_\mathcal{H} \\
\vdots \\
\langle f, g^k \rangle_\mathcal{H} \\
\end{bmatrix}
.
\end{equation}

Function encoders are useful because they allow for efficient adaptation to new functions while remaining highly expressive. For more information on this algorithm, see \citet{fe}.
% We will use this algorithm as a starting point for the methods below. 

\section{Methods}

\subsection{Problem Setting}

In this section, we extend the flow matching formulation to a \textit{collection} of distributions. Rather than a single target distribution, we are given a finite family of training distributions, each observed through samples. Our goal is to learn a model that can generate samples from any of these distributions.
Moreover, the goal is adaptation: after training, given samples from a previously unseen distribution, the model should efficiently adjust and generate samples from this new distribution.

To formalize this setting, let $\{p_X^\iota\}_{\iota \in \mathcal{I}}$ be a set of target data distributions. 
Assume we only observe a finite subset of the distributions in $\mathcal{I}$ during training. That is, there exists a training subset $\mathcal{T} \subset \mathcal{I}$.
We denote a dataset of independent samples from $p_X^\iota$ by $\mathcal{D}^\iota := \{x^{\iota,i}\}_{i=1}^m$, and assume such a dataset is available for each $\iota \in \mathcal{T}$; i.e., we have access to $\{\mathcal{D}^\iota\}_{\iota \in \mathcal{T}}$.
The goal is to train a model that can generate samples from each training distribution $p_X^\iota$. After training, our model should generalize to a new distribution $p_X^\iota$, $\iota \in \mathcal{I} \setminus \mathcal{T}$, given an unseen set of samples $x^{\iota,1}, ..., x^{\iota,m} \overset{i.i.d.}{\sim} p_X^\iota$. 
This implies our learned model conditions its output distribution on the samples $x^{\iota,1}, ..., x^{\iota,m}$. As an example, the most obvious solution to this problem is to finetune a model based on this new dataset.

\subsection{Approach} \label{sec:approach}

We begin by highlighting three observations that suggest a function encoder-style algorithm is well-suited to this problem. Firstly, for each $p_X^\iota$, we may define a stochastic process $X_t^\iota$ and its corresponding velocity field $v^\iota$ as in Section \ref{sec: fm}. 
Thus, we may view this problem as fitting velocity fields in the set $\mathcal{V} = \{v^\iota: \mathcal{X} \times [0,1] \to \mathbb{R}^n | \iota \in \mathcal{I}\}$.
Secondly, a crucial detail in the flow matching framework is that the loss function must be mean squared error in order to learn $\mathbb{E}[X_1 - X_0|X_t=x]$.
Mean squared error, in turn, can be viewed as the (squared) norm induced by the distribution-weighted inner product, as highlighted in section \ref{sec: fe}. Thus, it is natural to equip $\mathcal{V}$ with the distribution-weighted inner product as this maintains the geometry present in the flow matching framework. 
Thirdly,  access to the dataset $\mathcal{D}^\iota$ allows us to approximate inner products involving  the unobserved $v^\iota$ and a known function $g$:
\begin{align}
\langle v^\iota, g \rangle_{p_{t,X^\iota_t}} 
&= \mathbb{E}_{t,X_t^\iota} \left[ v^\iota(X^\iota_t,t)^\top g(X^\iota_t, t) \right] \\
&= \mathbb{E}_{t, X_t^\iota } \left[ \mathbb{E}_{X_1,X_0|X_t}\left[X^\iota_1 - X^\iota_0 \mid X^\iota_t \right]^\top g(X^\iota_t,t) \right] \label{eq:approx_ip}\\
&= \mathbb{E}_{t, X_1^\iota, X^\iota_0} \left[ (X^\iota_1 - X^\iota_0)^\top g((1-t)X_0^\iota + tX_1^\iota,t) \right] \label{eq:tower_property}
\end{align}
% Line \eqref{eq:tower_property} is a result of applying the tower property of conditional expectations. 
% It is clear we can sample $t, X_1^\iota,$ and $X_0^\iota$ as in Section \ref{sec: fm}. 
Equation \eqref{eq:tower_property} allows us to avoid sampling $X_t^\iota$ directly and does not require access to $v^\iota$. 
Therefore, we may approximate inner products involving $v^\iota$ under the current data assumptions. 
With these three observations taken together, it is clear this setting fits neatly into the function encoder framework.
In the following sections, we will introduce three variations of our method building upon these ideas. 

\subsubsection{Static FP-FM}

As suggested by the previous section, the most obvious approach is to directly apply the function encoder algorithm to this setting. That is, we approximate the velocity field $v^\iota$ as 
\begin{equation} \label{cfe-def}
    v^\iota(x,t) \approx \sum_{i=1}^k c^{\iota, i} g^i(x,t),
\end{equation}
where $\{g^i\}_{i=1}^k$ are neural networks and $c^\iota \in \mathbb{R}^k$. As the notation implies, for each velocity field $v^\iota$, there is a corresponding vector $c^\iota$. We may calculate the coefficients using least squares as above:
\begin{equation} \label{eqn: cfe-ls}
c^\iota = \left( \begin{bmatrix}
\langle g^1, g^1 \rangle_{p_{t,X_t^\iota}} & \hdots & \langle g^1, g^k \rangle_{p_{t,X_t^\iota}} \\
\vdots & \ddots & \vdots \\
\langle g^k, g^1 \rangle_{p_{t,X_t^\iota}} & \hdots & \langle g^k, g^k \rangle_{p_{t,X_t^\iota}} \\
\end{bmatrix} + \lambda I\right)^{-1}
\begin{bmatrix}
\langle v^\iota, g^1 \rangle_{p_{t,X_t^\iota}} \\
\vdots \\
\langle v^\iota, g^k \rangle_{p_{t,X_t^\iota}} \\
\end{bmatrix}.
\end{equation}
The key detail is that the distribution used to compute the inner product is $p_{t,X_t^\iota}$, the distribution of $X_t$'s corresponding to a specific target distribution $p_X^\iota$. Interestingly, this implies there is a separate inner product for each target distribution. This poses interesting questions from a mathematical perspective, which we discuss in Appendix \ref{app:function-distribution-pairs}. For the purposes of flow matching, the key detail is that we want one set of basis functions to span all of the target velocity fields, but we use a distribution-specific inner product to find the best approximation for a given velocity field within this basis. 
% The intuition is that for any specific velocity field $v^\iota$, we want to approximate $v^\iota$ as closely as possible for states that are likely under $X_t^\iota$, and we care less about the approximation error for states that are unlikely under $X_t^\iota$. 
To train the basis functions, we leverage the function encoder algorithm \citep{fe}, where a different inner product is used to compute each $c^\iota$. See Appendix \ref{app:algs}. 

We refer to this variant of the algorithm as \static. So long as we use a sufficient number of sufficiently expressive basis functions (see \citet{fe}), we may reasonably approximate the velocity fields of the training functions, $\{v^\iota  | \iota \in \mathcal{T} \}$.  Thus, the constant function encoder may generate new samples that are distributed approximately the same as any distribution $p_X^\iota$ in its training set. However, for $p_X^\iota, \iota \in \mathcal{I} \setminus \mathcal{T}$, problems arise. Using Equation \eqref{eqn: cfe-ls} yields the best approximation of any new velocity field $v_\iota$ within the linear span of the basis. However, the continuity equation is nonlinear, that is, the relationship between a distribution $p_X^\iota$ and its velocity field $v^\iota$ is nonlinear. 
% Therefore, even if we have trained the basis to span the velocity fields of the training set, a mixture distribution  is not necessarily within the span of our basis. 
% Thus, the constant function encoder may struggle to generalize even to distributions we intuitively expect should be ``easy''. 
Therefore, even if the basis spans the velocity fields in the training set, an ``easy'' distribution such as a mixture is not necessarily within the span of the basis. 
In the following sections, we will introduce modifications to this approach that address this issue.

\subsubsection{Temporal FP-FM}

To make FP-FM transfer to new distributions, we must correct the nonlinearity issue. The approximation of a given velocity field depends on two parts: The basis functions and the coefficients. The basis functions themselves are already highly expressive, so instead we modify the coefficient calculation. A straightforward way of doing so is to make the coefficients dependent on time. 
Let us first introduce a variation of the notation. 
For a given time $t$, and any function $h: \mathcal{X} \times [0,1] \to \mathbb{R}^n$, let the function at a fixed time $t$ be denoted as $h_t(\cdot) :=h(\cdot, t)$.
% This notation helps change our understanding of the velocity fields (and basis functions) from functions mapping states and times to velocities as instead a time-indexed family of functions mapping state to velocity. 
Then we can modify the equations above: 
\begin{equation} \label{tfe-def}
    v^\iota_t(x) \approx \sum_{i=1}^k c^{\iota,i}(t) g^i_t(x).
\end{equation}
To calculate the coefficients, we must also change the inner product:
\begin{equation} \label{eqn: tfe-ls}
c^\iota(t) = \left(\begin{bmatrix}
\langle g_t^1, g_t^1 \rangle_{p_{X_t^\iota|t}} & \hdots & \langle g_t^1, g_t^k \rangle_{p_{X_t^\iota|t}} \\
\vdots & \ddots & \vdots \\
\langle g_t^k, g_t^1 \rangle_{p_{X_t^\iota|t}} & \hdots & \langle g_t^k, g_t^k \rangle_{p_{X_t^\iota|t}} \\
\end{bmatrix} + \lambda I \right )^{-1}
\begin{bmatrix}
\langle v^\iota_t, g_t^1 \rangle_{p_{X_t^\iota|t}} \\
\vdots \\
\langle v^\iota_t, g_t^k \rangle_{p_{X_t^\iota|t}} \\
\end{bmatrix}.
\end{equation}
Note that the distribution for the inner product is $p_{X_t^\iota|t}$; that is, it depends on the distribution of states at that specific time $t$. While this change appears small in notation, its impact is profound. We offer three separate ways to understand the significance, starting with the concrete and moving towards the abstract. Firstly, during generation, we must calculate the coefficients at every time $t$, rather than only once. Secondly, we can now view the coefficients as a function $c^\iota: [0,1] \to \mathbb{R}^k$, meaning that the coefficients are now a function of time. Thirdly, we can view this as modifying the function space definition. 
Where before we were training basis functions to span $\{v^\iota: \mathcal{X} \times [0,1] \to \mathbb{R}^n | \iota \in \mathcal{I}\}$, we are now training basis functions to span $\{v^\iota_t: \mathcal{X} \to \mathbb{R}^n | \iota \in \mathcal{I}, t \in [0,1]\}$. Where before we had $k$ basis functions $\{g^i\}_{i=1}^k$,  we now have $k$ basis functions at every time $t$, i.e., $\{g_t^i\}_{i=1}^k$.

% \temporal   is strictly more expressive than \static. For example, \temporal   can recover \static   if $c^\iota(t)$ is constant with respect to $t$. Thus, \temporal   encoder may approximate a greater set of velocity fields than \static, and therefore generalizes to a larger set of target distributions. Later on, we will discuss the tradeoffs in more detail. 

\subsubsection{Dynamic FP-FM}

In this section, we take this idea to its natural conclusion. If we can make the model more expressive by making its coefficients depend on time, then we can further increase the expressivity by making the coefficients depend on the state:
\begin{equation} \label{stfe-def}
    v^\iota(x,t) \approx \sum_{i=1}^k c^{\iota, i}(x,t) g^i(x,t).
\end{equation}
To compute the localized coefficients $c^\iota(x,t)$, we evaluate the inner product defined previously, but condition it on the observed state $X_t=x$ and time $t$. Denoting this conditional inner product as $\langle \cdot, \cdot \rangle_{x,t}$, the coefficients are computed as the solution to the localized least-squares problem:
\begin{equation} \label{eqn: stfe-ls}
c^\iota(x,t) = \left ( \begin{bmatrix}
\langle g^1, g^1 \rangle_{x,t} & \hdots & \langle g^1, g^k \rangle_{x,t} \\
\vdots & \ddots & \vdots \\
\langle g^k, g^1 \rangle_{x,t} & \hdots & \langle g^k, g^k \rangle_{x,t} \\
\end{bmatrix} + \lambda I \right ) ^{-1}
\begin{bmatrix}
\langle v^\iota, g^1 \rangle_{x,t} \\
\vdots \\
\langle v^\iota, g^k \rangle_{x,t} \\
\end{bmatrix}.
\end{equation}
Because $x$ and $t$ are fixed, the conditional expectation collapses. Thus, we can directly compute the left-hand side as $\langle g^i, g^j\rangle_{x,t} = \frac{1}{n}g^i(x,t)^\top g^j(x,t)$. However, we still do not have direct access to the true velocity field $v^\iota$, meaning we must still use Equation \eqref{eq:approx_ip}. By exploiting the conditional distribution, the right-hand side simplifies to $\langle v^\iota, g^i \rangle_{x,t} = \frac{1}{n} \mathbb{E}_{X_1,X_0|X_t}\big[X^\iota_1 - X^\iota_0 \mid X^\iota_t=x \big]^\top g^i(x,t)$, which we again approximate from empirical data. Putting these two facts together, solving Equation \eqref{eqn: stfe-ls} effectively performs two steps: 1) computing an empirical estimate of the mean velocity field given $X_t^\iota=x$, and 2) finding the best approximation of this vector within the $k$-dimensional subspace of $\mathbb{R}^n$ spanned by the basis $\{g^i(x,t)\}_{i=1}^k$.
We can interpret this as similar to a purely data-driven approach, except the learned functions $\{g^i\}$ play the crucial role of regularizing the approximation based on the data present in the training set.
% \footnote{Equation \eqref{eqn: stfe-ls} projects a $n$-dimensional vector onto a $k$-dimensional subspace of $\mathbb{R}^n$. }
% .

% Another important consideration is the approximation of $\mathbb{E}_{X_1,X_0|X_t}\left[X^\iota_1 - X^\iota_0 \mid X^\iota_t=x \right]$. Since the expectation is conditioned on $X_t^\iota=x$, we can no longer approximate this expectation by independently sampling $X_0^\iota$ and $X_1^\iota$, as we must ensure $X_t^\iota = (1-t)X_0^\iota + t X_1^\iota$. Instead, we make use of the following Theorem:
% \begin{theorem} \label{thm: sampling}
%     Given $X_t=x$, and the random variable $X_1$, define $X_0^*=\frac{X_t - tX_1}{1-t}$. Then 
%     \[    \mathbb{E}_{X_1,X_0|X_t}\left[X_1 - X_0 \mid X_t=x \right] = E_{X_1}\left[(X_1 - X^*_0) \frac{ p(X^*_0)}{E_{X_1}\left[p\left(X^*_0\right)\right]}\right]. \]
% \end{theorem}
% This formulation is convenient for Monte Carlo approximation. We already have samples from the target distribution, and computing $X_0^*$ is straightforward given $X_t = x$ and $X_1$. Moreover, because the noise distribution is Gaussian, $p(X_0^*)$ can be evaluated directly. These properties together make it easy to construct an accurate estimator of the conditional expectation. A derivation of the theorem and a detailed description of the sampling procedure are provided in Appendix \ref{thm: sampling}.

Another important consideration is approximating the conditional expectation $\mathbb{E}[X^\iota_1 - X^\iota_0 \mid X^\iota_t=x]$. Because this expectation is conditioned on the fixed state, we can no longer estimate it by drawing independent samples of $X_0^\iota$ and $X_1^\iota$; any valid pair must strictly satisfy the deterministic constraint $X_t^\iota = (1-t)X_0^\iota + t X_1^\iota$. Instead, we bypass this issue using the following general result:

\begin{theorem} \label{thm: sampling}
    Given the fixed noisy example $X_t=x$ and a random variable $X_1$, let $X_0^*=\frac{x - tX_1}{1-t}$. The conditional expectation can be expressed as an expectation over the marginal distribution of $X_1$:
    \begin{equation*}
        \mathbb{E}[X_1 - X_0 \mid X_t=x] = \mathbb{E}_{X_1}\left[(X_1 - X^*_0) \frac{ p(X^*_0)}{\mathbb{E}_{X_1'}[p(X^{*\prime}_0)]}\right]. 
    \end{equation*}
\end{theorem}

This formulation is highly amenable to Monte Carlo approximation. We already possess samples from the target distribution $p(X_1)$, and computing $X_0^*$ is a trivial linear operation given $x$ and $X_1$. Furthermore, assuming the base noise distribution is Gaussian, $p(X_0^*)$ can be evaluated analytically. Together, these properties yield an efficient, self-normalized importance sampling estimator for the conditional expectation. Furthermore, the basis functions play the crucial role of preventing this estimator from overfitting to a single sample, since
high-dimensional Gaussian importance weights tend to collapse when $t \approx 1$.
Because the expectation is projected onto the basis, it cannot simply move towards the most likely sample directly; it instead must move in a direction which is generally good for generating images, as stored in the basis functions. 
A full derivation of the theorem and algorithmic details regarding the sampling procedure are deferred to Appendix \ref{app: sampling}.

\subsubsection{Tradeoffs}

Next, we will discuss the tradeoffs between these three approaches. As suggested in the preceding sections, we observe an ordering in expressive power: \dynamic   exhibits the greatest expressivity, followed by the \temporal, and finally \static. 
% Note that flow matching inherits much of the intuition from standard regression, and in that sense \dynamic  is also most susceptible to overfitting, followed by \temporal, and then \static. 
On the other hand, we observe the opposite relationship in terms of compute. \static   is the most compute efficient, as it computes the coefficients only once per distribution. \temporal   and \dynamic require more compute, albeit with different scaling properties. \temporal requires less least-squares solves, but more forward passes than \dynamic, since it needs to forward pass the basis at the current $\hat{x}_t$ and also at the empirical $x_t$'s. Meanwhile, \dynamic requires a least-squares solve for every time and sample, whereas \temporal only requires one least-squares solve at each time. 
In the following sections, we will empirically compare the three approaches to validate these tradeoffs. 

% Usage: \qualimg{<baseline>_<dataset>_<dist>.png}
% Example: \qualimg{Unconditional_MNIST_TD.png}

\begin{figure*}[!h]
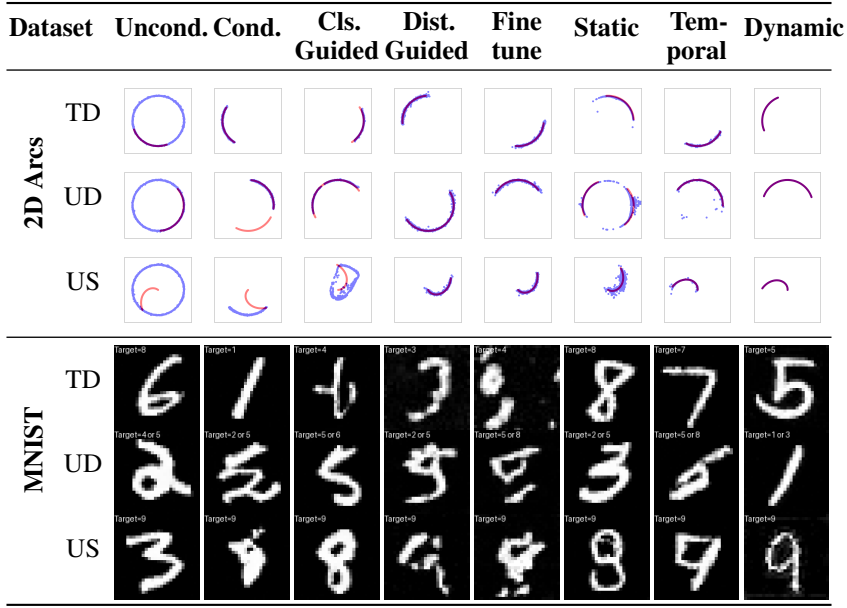

\centering
\setlength{\tabcolsep}{1pt}
\renewcommand{\arraystretch}{0.0}

\begin{tabular}{C{0.04\textwidth} C{0.05\textwidth}  *{8}{C{0.08\textwidth}}}
\toprule
\textbf{Dataset} &  
& \textbf{Uncond.}
& \textbf{Cond.}
& \textbf{Cls. Guided}
& \textbf{Dist. Guided}
& \textbf{Fine tune}
& \textbf{Static}
& \textbf{Tem-poral}
& \textbf{Dynamic} \\
\midrule

\multirow[c]{3}{*}{\vspace{-50pt} \rotatebox{90}{\textbf{2D Arcs}}}
& TD
& \qualimg{Unconditional_2DArcs_TD.png}
& \qualimg{Conditional_2DArcs_TD.png}
& \qualimg{ClassifierGuided_2DArcs_TD.png}
& \qualimg{DistributionGuided_2DArcs_TD.png}
& \qualimg{Finetune_2DArcs_TD.png}
& \qualimg{ConstantFE_2DArcs_TD.png}
& \qualimg{TimeFE_2DArcs_TD.png}
& \qualimg{StateTimeFE_2DArcs_TD.png} \\
& UD
& \qualimg{Unconditional_2DArcs_UD.png}
& \qualimg{Conditional_2DArcs_UD.png}
& \qualimg{ClassifierGuided_2DArcs_UD.png}
& \qualimg{DistributionGuided_2DArcs_UD.png}
& \qualimg{Finetune_2DArcs_UD.png}
& \qualimg{ConstantFE_2DArcs_UD.png}
& \qualimg{TimeFE_2DArcs_UD.png}
& \qualimg{StateTimeFE_2DArcs_UD.png} \\
& US
& \qualimg{Unconditional_2DArcs_US.png}
& \qualimg{Conditional_2DArcs_US.png}
& \qualimg{ClassifierGuided_2DArcs_US.png}
& \qualimg{DistributionGuided_2DArcs_US.png}
& \qualimg{Finetune_2DArcs_US.png}
& \qualimg{ConstantFE_2DArcs_US.png}
& \qualimg{TimeFE_2DArcs_US.png}
& \qualimg{StateTimeFE_2DArcs_US.png} \\
\midrule

\multirow[c]{3}{*}{\vspace{-50pt}  \rotatebox{90}{\textbf{MNIST}}}
& TD
& \qualimg{Unconditional_MNIST_TD.png}
& \qualimg{Conditional_MNIST_TD.png}
& \qualimg{ClassifierGuided_MNIST_TD.png}
& \qualimg{DistributionGuided_MNIST_TD.png}
& \qualimg{Finetune_MNIST_TD.png}
& \qualimg{ConstantFE_MNIST_TD.png}
& \qualimg{TimeFE_MNIST_TD.png}
& \qualimg{StateTimeFE_MNIST_TD.png} \\
& UD
& \qualimg{Unconditional_MNIST_UD.png}
& \qualimg{Conditional_MNIST_UD.png}
& \qualimg{ClassifierGuided_MNIST_UD.png}
& \qualimg{DistributionGuided_MNIST_UD.png}
& \qualimg{Finetune_MNIST_UD.png}
& \qualimg{ConstantFE_MNIST_UD.png}
& \qualimg{TimeFE_MNIST_UD.png}
& \qualimg{StateTimeFE_MNIST_UD.png} \\
& US
& \qualimg{Unconditional_MNIST_US.png}
& \qualimg{Conditional_MNIST_US.png}
& \qualimg{ClassifierGuided_MNIST_US.png}
& \qualimg{DistributionGuided_MNIST_US.png}
& \qualimg{Finetune_MNIST_US.png}
& \qualimg{ConstantFE_MNIST_US.png}
& \qualimg{TimeFE_MNIST_US.png}
& \qualimg{StateTimeFE_MNIST_US.png} \\
% \midrule

% \multirow[c]{3}{*}{\vspace{-80pt}  \rotatebox{90}{\textbf{ImageNet}}}
% & TD
% & \qualimg{Unconditional_ImageNet_TD.png}
% & \qualimg{Conditional_ImageNet_TD.png}
% & \qualimg{ClassifierGuided_ImageNet_TD.png}
% & \qualimg{DistributionGuided_ImageNet_TD.png}
% & \qualimg{Finetune_ImageNet_TD.png}
% & \qualimg{ConstantFE_ImageNet_TD.png}
% & \qualimg{TimeFE_ImageNet_TD.png}
% & \qualimg{StateTimeFE_ImageNet_TD.png} \\
% & UD
% & \qualimg{Unconditional_ImageNet_UD.png}
% & \qualimg{Conditional_ImageNet_UD.png}
% & \qualimg{ClassifierGuided_ImageNet_UD.png}
% & \qualimg{DistributionGuided_ImageNet_UD.png}
% & \qualimg{Finetune_ImageNet_UD.png}
% & \qualimg{ConstantFE_ImageNet_UD.png}
% & \qualimg{TimeFE_ImageNet_UD.png}
% & \qualimg{StateTimeFE_ImageNet_UD.png} \\
% & US
% & \qualimg{Unconditional_ImageNet_US.png}
% & \qualimg{Conditional_ImageNet_US.png}
% & \qualimg{ClassifierGuided_ImageNet_US.png}
% & \qualimg{DistributionGuided_ImageNet_US.png}
% & \qualimg{Finetune_ImageNet_US.png}
% & \qualimg{ConstantFE_ImageNet_US.png}
% & \qualimg{TimeFE_ImageNet_US.png}
% & \qualimg{StateTimeFE_ImageNet_US.png} \\
\bottomrule
\end{tabular}
\caption{
Qualitative comparison across datasets and distribution types. For 2D Arcs, the image shows the full distributions, where red points correspond to samples from the target distribution and blue points are samples from the model. For MNIST, each image shows one sample from the model. 
}
\label{fig:qualitative_table}
\end{figure*}

\section{Experiments}

% \subsection{Design}

% \paragraph{Datasets.} 
We validate our approach on three datasets: \textit{2D Arcs}, \textit{MNIST} \citep{mnist}, and \textit{ImageNet} \citep{ImageNet}. 2D Arcs is a low-dimensional dataset that enables visualization of the distributions. MNIST evaluates performance on structured image data, while ImageNet tests scalability to large-scale, high-dimensional data.
To measure the generalization capabilities of these approaches, we evaluate on three distribution splits for each dataset. \textit{Training Distribution (TD)} is a distribution that was seen during training. \textit{Unseen Distribution (UD)} is a distribution that is not in the training set, but the samples have been seen. For example, if we train on MNIST on digits 0 through 8, an unseen distribution could be a mixture of 0 and 1. \textit{Unseen Support (US)} is an unseen distribution with a different support than what was seen during training. To continue the analogy, if we train on MNIST 0 through 8, the distribution corresponding to 9 has an unseen support.

% \paragraph{Baselines.} 
We compare against six baselines. \unconditional   is a flow matching model \citep{fm} trained on all training distributions as if they were a single distribution. This acts as naive baseline, and as a starting point for some of the other algorithms. \conditional   is a model provided with a conditioning variable to describe the target distribution, e.g., a one-hot encoding of the class \citep{classifierfree}. 
\cg   trains a classifier for a given target distribution, then uses the score to guide generation of a pretrained, unconditional model \citep{classifierguided}. \dg   starts with a pretrained, unconditional model. It then trains a secondary model to produce noise inputs that, when passed through the pretrained model, generate samples from the target distribution. See Appendix \ref{app:baselines} for a visualization.  \finetune   is the standard approach. Given a pretrained unconditional model, finetune it on the data provided for the target distribution. 
Note that there are many tradeoffs you can make during finetuning; e.g., LoRA \citep{hu2021lora} is more compute efficient but less expressive.
\reference is an oracle baseline consisting of a small set of real images compared against a different set of real images, and acts as an approximate optimal performance. 
For more detailed descriptions of the baselines, see Appendix \ref{app:baselines}.

% \paragraph{Metrics.}
To evaluate these approaches, we use various metrics: \textit{Time} measures how long it takes to generate a batch of samples.
% \footnote{We use a batch of samples to measure generation time to distinguish between \temporal   and \dynamic. If the batch size were 1, these two approaches would have the same generation time.}
\textit{Precision} measures how well the generated samples match the target distribution; \textit{Recall} measures how well real samples match the generated distribution \citep{precisionrecall}. For the image-based datasets, we use \textit{FID} \citep{fid}.
For implementation details, see Appendix \ref{app:baselines}. 
Results are shown in Figure \ref{fig:qualitative_table} and in Tables \ref{tab:arc_quant_table_exclude}, \ref{tab:mnist_quant_table_exclude}, and \ref{table:imagenet}.

\begin{table*}[t]
\centering
\small
\setlength{\tabcolsep}{4pt}
\begin{tabular}{l|rrr|rrr|r}
\toprule
& \multicolumn{3}{c|}{Precision $\uparrow$} & \multicolumn{3}{c|}{Recall $\uparrow$} & Time (s) $\downarrow$ \\
Algorithm & TD & UD & US & TD & UD & US & \\
\midrule
\unconditional & 0.108 & 0.291 & 0.003 & 0.997 & 0.999 & 0.032 & 0.17 \\
\conditional & 0.369 & 0.608 & 0.006 & 0.998 & 0.613 & 0.041 & 0.17 \\
\cg & 0.407 & 0.654 & 0.009 & 0.836 & 0.867 & 0.134 & 7.08 \\
\dg & 0.980 & 0.980 & 0.016 & 0.999 & 0.999 & 0.997 & 16.48 \\
\finetune & 0.612 & 0.735 & 0.385 & 0.998 & 0.999 & 0.999 & 4.77 \\
\static & 0.324 & 0.320 & 0.086 & 0.847 & 0.811 & 0.894 & 0.21 \\
\temporal & 0.586 & 0.763 & 0.709 & 0.997 & 0.997 & 0.998 & 0.72 \\
\dynamic & 0.991 & 0.997 & 0.934 & 0.989 & 0.985 & 0.972 & 10.99 \\
\reference & 0.999 & 0.999 & 0.999 & 0.999 & 0.999 & 0.999 & - \\
\bottomrule
\end{tabular}
\caption{\textbf{2D Arc.} Comparison of methods across TD, UD, and US settings. Arrows indicate whether higher ($\uparrow$) or lower ($\downarrow$) is better. Results are averaged over 5 seeds. }
\label{tab:arc_quant_table_exclude}
\end{table*}

\begin{table*}[t]
\centering
\small
\setlength{\tabcolsep}{4pt}
\begin{tabular}{l|rrr|rrr|rrr|r}
\toprule
& \multicolumn{3}{c|}{Precision $\uparrow$} & \multicolumn{3}{c|}{Recall $\uparrow$} & \multicolumn{3}{c|}{FID $\downarrow$} & Time $\downarrow$ \\
Algorithm & TD & UD & US & TD & UD & US & TD & UD & US &  \\
\midrule
\unconditional & 0.342 & 0.464 & 0.303 & 0.705 & 0.705 & 0.666 & 97.93 & 88.59 & 92.84 & 10.31 \\
\conditional & 0.875 & 0.763 & 0.309 & 0.877 & 0.759 & 0.783 & 12.54 & 23.18 & 31.92 & 10.78 \\
\cg & 0.582 & 0.632 & 0.440 & 0.719 & 0.720 & 0.714 & 82.00 & 80.18 & 88.98 & 29.19 \\
\dg & 0.753 & 0.725 & 0.686 & 0.781 & 0.760 & 0.734 & 18.39 & 21.99 & 15.71 & 626.17 \\
\finetune & 0.728 & 0.728 & 0.737 & 0.590 & 0.606 & 0.566 & 80.31 & 81.37 & 81.45 & 215.71 \\
\static & 0.783 & 0.757 & 0.554 & 0.865 & 0.856 & 0.766 & 16.71 & 21.03 & 22.53 & 10.63 \\
\temporal & 0.771 & 0.767 & 0.569 & 0.885 & 0.897 & 0.797 & 14.13 & 17.49 & 17.69 & 20.78 \\
\dynamic & 0.759 & 0.829 & 0.723 & 0.818 & 0.817 & 0.837 & 17.38 & 18.38 & 18.37 & 21.96 \\
\reference & 0.891 & 0.898 & 0.900 & 0.897 & 0.893 & 0.903 & 11.16 & 12.91 & 10.27 & - \\
\bottomrule
\end{tabular}
\caption{\textbf{MNIST.} Comparison of methods across TD, UD, and US settings. Arrows indicate whether higher ($\uparrow$) or lower ($\downarrow$) is better. Results are averaged over 5 seeds. }
\label{tab:mnist_quant_table_exclude}
\end{table*}

\subsection{Results}

\paragraph{2D Arcs.}
The TDs are one quarter of the unit circle  by central angle, the UDs are mixtures of two training distributions, and the US is a spiral from the origin to the unit circle. Each baseline is trained for 10,000 gradient steps; see Figure \ref{fig:qualitative_table}. For brevity, standard deviations are in the appendix. 

Qualitatively, \static fits TDs but fails on UD and US, while \temporal generalizes to all three settings and \dynamic produces the most accurate reconstructions. Among the baselines, \unconditional fails to capture specific targets, \conditional and \cg do not generalize to novel supports, and \dg and \finetune approximate all settings with noticeable error. Consistent with these observations, many methods achieve high recall but low precision (especially on US), indicating overestimated support, whereas \dynamic achieves the highest precision on UD and US. \temporal also generalizes well to US, but is not as precise as \dynamic.

\paragraph{MNIST.}
The TDs consist of digit classes 0 through 8. UDs are defined as mixtures of any two of these classes, while US corresponds to the unseen class 9. We train each baseline for 20,000 gradient steps. 
Qualitatively, the trends in 2D Arcs carry over to MNIST. \static successfully models TDs but fails to generalize to UD and US. \temporal improves generalization to mixtures and unseen classes, producing reasonable samples across all settings. \dynamic provides the most consistent performance, performing especially well on UD and US.
These trends are reflected in the quantitative metrics. 
Other methods tend to achieve reasonable recall but significantly lower precision and worse FID scores.
\dg has high precision even on US, but this comes at the cost of a significantly increased generation time.
We  also use MNIST to perform ablations on the number of shots and the number of basis functions.
See Appendix \ref{app:ablation}.

\paragraph{ImageNet.}
Similar to MNIST, we use a random subset with 900 classes out of 1000 to form the TDs, and the rest are used for US. 
We use a latent Diffusion Transformer (DiT) \citep{vit, Peebles2022DiT} backbone and each baseline is trained for 150,000 training steps. 
Due to the expense of training, we only train one model and focus on key baselines. 
See Figure \ref{fig:method-nearest-neighbors} and Table \ref{table:imagenet}.
Quantitatively, \dynamic achieves the highest precision and recall on US, as expected. Interestingly, \temporal performs comparably to \finetune despite requiring roughly 30$\times$ less compute. Qualitatively, \dynamic largely reconstructs the images provided in context, introducing only minor visual perturbations. A particularly striking pattern emerges between \temporal and \finetune: across all six classes, the two methods generate highly similar images. Because all methods are initialized from the same random noise, this similarity suggests that \temporal and \finetune learn closely aligned conditional expectations. These results indicate that \temporal may provide a computationally efficient alternative to finetuning while producing comparable samples.

% Qualitatively, \conditional, \finetune, and \dynamic successfully capture TDs, albeit with minor artifacts. While \static and \temporal also model these distributions, their representations are comparatively limited here, likely because a memory-constrained batch size of 32 is too small to fully capture the class-specific context.
% On US, \conditional unsurprisingly fails. \finetune's outputs lack key details, although the low frequency features look correct. \dynamic outputs the best images. 
% These results are confirmed by the quantitative metrics: \dynamic has the best precision, recall, and FID. 
% \conditional is notably good for the TDs, but it cannot generalize. 

% \documentclass[10pt]{article}

% \usepackage[margin=0.35in]{geometry}
% \usepackage{graphicx}
% \usepackage{booktabs}

% These fallbacks make this file compile on its own. If the paper already
% defines these commands, its definitions are used instead.
% \providecommand{\conditional}{Conditional}
% \providecommand{\finetune}{Fine-tune}
% \providecommand{\static}{Static}
% \providecommand{\temporal}{Temporal}
% \providecommand{\dynamic}{Dynamic}
\newcommand{\imagefolder}{overleaf_images}
\newcommand{\genimage}[1]{%
  \includegraphics[width=0.080\textwidth]{#1/gen_00.jpg}%
}
\newcommand{\nnimage}[1]{%
  \includegraphics[width=0.080\textwidth]{#1/gen_00_nn.jpg}%
}

\begin{figure*}[t]
  \centering
  \setlength{\tabcolsep}{0.5pt}
  \renewcommand{\arraystretch}{0.3}
  \begin{tabular}{
  @{}
  >{\raggedright\arraybackslash}m{0.115\textwidth}
  *{10}{>{\centering\arraybackslash}m{0.080\textwidth}}
  @{}
}
    & \multicolumn{2}{c}{\scriptsize\conditional}
    & \multicolumn{2}{c}{\scriptsize\finetune}
    & \multicolumn{2}{c}{\scriptsize\static}
    & \multicolumn{2}{c}{\scriptsize\temporal}
    & \multicolumn{2}{c}{\scriptsize\dynamic} \\
    \cmidrule(lr){2-3}\cmidrule(lr){4-5}\cmidrule(lr){6-7}
    \cmidrule(lr){8-9}\cmidrule(lr){10-11}
    & \scriptsize Gen. & \scriptsize NN
    & \scriptsize Gen. & \scriptsize NN
    & \scriptsize Gen. & \scriptsize NN
    & \scriptsize Gen. & \scriptsize NN
    & \scriptsize Gen. & \scriptsize NN \\
    \midrule
    tabby
    & \genimage{\imagefolder/c0281_tabby/conditional} & \nnimage{\imagefolder/c0281_tabby/conditional}
    & \genimage{\imagefolder/c0281_tabby/finetune_lora} & \nnimage{\imagefolder/c0281_tabby/finetune_lora}
    & \genimage{\imagefolder/c0281_tabby/static_fpfm} & \nnimage{\imagefolder/c0281_tabby/static_fpfm}
    & \genimage{\imagefolder/c0281_tabby/temporal_fpfm} & \nnimage{\imagefolder/c0281_tabby/temporal_fpfm}
    & \genimage{\imagefolder/c0281_tabby/dynamic_fpfm} & \nnimage{\imagefolder/c0281_tabby/dynamic_fpfm} \\
    envelope
    & \genimage{\imagefolder/c0549_envelope/conditional} & \nnimage{\imagefolder/c0549_envelope/conditional}
    & \genimage{\imagefolder/c0549_envelope/finetune_lora} & \nnimage{\imagefolder/c0549_envelope/finetune_lora}
    & \genimage{\imagefolder/c0549_envelope/static_fpfm} & \nnimage{\imagefolder/c0549_envelope/static_fpfm}
    & \genimage{\imagefolder/c0549_envelope/temporal_fpfm} & \nnimage{\imagefolder/c0549_envelope/temporal_fpfm}
    & \genimage{\imagefolder/c0549_envelope/dynamic_fpfm} & \nnimage{\imagefolder/c0549_envelope/dynamic_fpfm} \\
    minibus
    & \genimage{\imagefolder/c0654_minibus/conditional} & \nnimage{\imagefolder/c0654_minibus/conditional}
    & \genimage{\imagefolder/c0654_minibus/finetune_lora} & \nnimage{\imagefolder/c0654_minibus/finetune_lora}
    & \genimage{\imagefolder/c0654_minibus/static_fpfm} & \nnimage{\imagefolder/c0654_minibus/static_fpfm}
    & \genimage{\imagefolder/c0654_minibus/temporal_fpfm} & \nnimage{\imagefolder/c0654_minibus/temporal_fpfm}
    & \genimage{\imagefolder/c0654_minibus/dynamic_fpfm} & \nnimage{\imagefolder/c0654_minibus/dynamic_fpfm} \\
    scuba diver
    & \genimage{\imagefolder/c0983_scuba_diver/conditional} & \nnimage{\imagefolder/c0983_scuba_diver/conditional}
    & \genimage{\imagefolder/c0983_scuba_diver/finetune_lora} & \nnimage{\imagefolder/c0983_scuba_diver/finetune_lora}
    & \genimage{\imagefolder/c0983_scuba_diver/static_fpfm} & \nnimage{\imagefolder/c0983_scuba_diver/static_fpfm}
    & \genimage{\imagefolder/c0983_scuba_diver/temporal_fpfm} & \nnimage{\imagefolder/c0983_scuba_diver/temporal_fpfm}
    & \genimage{\imagefolder/c0983_scuba_diver/dynamic_fpfm} & \nnimage{\imagefolder/c0983_scuba_diver/dynamic_fpfm} \\
    tiger beetle
    & \genimage{\imagefolder/c0300_tiger_beetle/conditional} & \nnimage{\imagefolder/c0300_tiger_beetle/conditional}
    & \genimage{\imagefolder/c0300_tiger_beetle/finetune_lora} & \nnimage{\imagefolder/c0300_tiger_beetle/finetune_lora}
    & \genimage{\imagefolder/c0300_tiger_beetle/static_fpfm} & \nnimage{\imagefolder/c0300_tiger_beetle/static_fpfm}
    & \genimage{\imagefolder/c0300_tiger_beetle/temporal_fpfm} & \nnimage{\imagefolder/c0300_tiger_beetle/temporal_fpfm}
    & \genimage{\imagefolder/c0300_tiger_beetle/dynamic_fpfm} & \nnimage{\imagefolder/c0300_tiger_beetle/dynamic_fpfm} \\
    mushroom
    & \genimage{\imagefolder/c0947_mushroom/conditional} & \nnimage{\imagefolder/c0947_mushroom/conditional}
    & \genimage{\imagefolder/c0947_mushroom/finetune_lora} & \nnimage{\imagefolder/c0947_mushroom/finetune_lora}
    & \genimage{\imagefolder/c0947_mushroom/static_fpfm} & \nnimage{\imagefolder/c0947_mushroom/static_fpfm}
    & \genimage{\imagefolder/c0947_mushroom/temporal_fpfm} & \nnimage{\imagefolder/c0947_mushroom/temporal_fpfm}
    & \genimage{\imagefolder/c0947_mushroom/dynamic_fpfm} & \nnimage{\imagefolder/c0947_mushroom/dynamic_fpfm} \\
    \bottomrule
  \end{tabular}
  \caption{Generated samples and their nearest neighbors in the context set.
  Within each method cell, the generated sample is shown on the left
  and its nearest neighbor is shown on the right. The same noise $x_0$ is used for all models. Nearest neighbors are calculated in embedding space. The classes are heldout. This causes \conditional to fail, since the class variable is unseen. 
  }
  \label{fig:method-nearest-neighbors}
\end{figure*}

% \begin{table*}[t]
% % \centering
% \small
% \setlength{\tabcolsep}{4pt}

% \begin{minipage}[c]{0.7\textwidth}
% \centering
% \begin{tabular}{l|rr|rr|rr|r}
% \toprule
% & \multicolumn{2}{c|}{Precision $\uparrow$} & \multicolumn{2}{c|}{Recall $\uparrow$} & \multicolumn{2}{c}{FID $\downarrow$} & \multicolumn{1}{c}
% {Time (s) $\downarrow$} \\
% Algorithm & TD & US & TD & US & TD & US & \\
% \midrule
% \conditional & 0.509 & 0.012 & 0.225 & 0.394 & 156.3 & 378.7 & 3.16\\
% \finetune (LoRA) & 0.489 & 0.344 & 0.143 & 0.251 & 162.6 & 193.8 & 448.0 \\
% \static      & 0.025 & 0.017 & 0.377 & 0.361 & 341.1 & 331.4 & todo\\
% \temporal    & 0.015 & 0.013 & 0.128 & 0.153 & 366.7 & 360.3 & todo\\
% \dynamic     & \textbf{0.669} & \textbf{0.699} & \textbf{0.667} & \textbf{0.682} & \textbf{126.1} & \textbf{117.0} & 4.01 \\
% \bottomrule
% \end{tabular}
% \end{minipage}
% \hfill
% \begin{minipage}[c]{0.23\textwidth}
% \captionof{table}{\textbf{Imagenet.} Comparison of methods across TD and US settings. Arrows indicate whether higher ($\uparrow$) or lower ($\downarrow$) is better.} \label{table:imagenet}
% \end{minipage}

% \end{table*}

\begin{table*}[t]
% \centering
\small
\setlength{\tabcolsep}{4pt}

\begin{minipage}[c]{0.7\textwidth}
\centering
\begin{tabular}{l|rr|rr|rr|r}
\toprule
& \multicolumn{2}{c|}{Precision $\uparrow$} & \multicolumn{2}{c|}{Recall $\uparrow$} & \multicolumn{2}{c}{FID $\downarrow$} & \multicolumn{1}{c}
{Time (s) $\downarrow$} \\
Algorithm & TD & US & TD & US & TD & US & \\
\midrule
\conditional       & 0.478 & 0.011 & 0.181 & 0.157 & 154.1 & 382.2 & 19.6 \\
\finetune (LoRA)   & 0.487 & 0.264 & 0.192 & 0.261 & 150.7 & 222.4 & 607.5 \\
\static            & 0.023 & 0.018 & 0.194 & 0.238 & 348.1 & 339.8 & 10.7 \\
\temporal          & 0.289 & 0.252 & 0.335 & 0.354 & 207.6 & 212.8 & 20.8 \\
\dynamic           & 0.412 & 0.433 & 0.541 & 0.545 & 136.6 & 130.2 & 14.1 \\
\reference           & 0.696 & 0.693 & 0.702 & 0.676 & 75.8  & 76.4  & -- \\
\bottomrule
\end{tabular}
\end{minipage}
\hfill
\begin{minipage}[c]{0.23\textwidth}
\captionof{table}{\textbf{Imagenet.} Comparison of methods across TD and US settings. Arrows indicate whether higher ($\uparrow$) or lower ($\downarrow$) is better. Time is measured for generating 128 samples. }
\label{table:imagenet}
\end{minipage}

\end{table*}

\section{Discussion} \label{sec:discussion}

Overall, the experiments show that \dynamic achieves the best precision, followed by \temporal and then \static. However, \dynamic tends to generate samples similar to the context, which limits its usefulness in many image-based settings. Nonetheless, it may find uses in other settings where only minor perturbations are needed.
On the hand, \temporal performs comparably to \finetune despite having much lower cost. We therefore suggests practitioners should consider \temporal for image-based in-context generation. 
% This is in alignment with the ordering of expressivity, as discussed in Section \ref{sec:approach}.
% Compared to other baselines, FP-FM has lower compute times than any method that requires finetuning, while paying a small cost relative to the conditional approach.
% \dynamic achieves competitive or superior performance to prior work, particularly on UDs and USs, while requiring lower compute than training-based approaches.

\paragraph{Future work. }One consideration is the desired expressivity of the model.
To that end, an direction is to generalize our method to a spectrum of methods with controllable expressivity.
For example, if the coefficients are thought of as a function of state and time, then imposing characteristics such as Lipschitz constraints control the expressivity of the model. 
In this sense, \static's coefficients are 0-Lipschitz with respect to state and time, while \dynamic's coefficients have no Lipschitz constraints. 
\temporal is 0-Lipschitz with respect to state but not time. 
Generalizing this method to a spectrum is an direction, and we leave it to future work. 

\paragraph{Limitations.} \label{para:limitations}
While FP-FM is significantly more computationally efficient than finetuning, it is slightly more expensive than a \conditional model, as shown in the experiments. This tradeoff is empirically justified by improvements in UD and US performance. 
FP-FM also requires samples, rather than a conditioning variable. 
This assumption may not be reasonable in every setting. 
Finally, performance improves with the number and quality of samples provided. In extremely low-sample regimes, the method may struggle to accurately capture the underlying manifold, and the required number of samples depends on the target distribution.

\section{Related Work}

FP-FM builds on parametric generative modeling \citep{ddpm, scored-based} and the ODE viewpoint underlying flow and rectified-flow models \citep{fm, rect-flow}, in which a velocity field transports noise to a target distribution via a neural ODE \citep{node}. Unlike class-conditioned or text-conditioned models \citep{conditional-norm-flows, classifierfree, dalle, imagen, stablediffusion, eDiff}, FP-FM conditions directly on samples, avoiding conditioning variables that may require extensive data or inadequately describe the target. It also adapts without the parameter updates used by fine-tuning methods \citep{dreambooth, vit} or the additional training required by classifier-guided and Langevin-based approaches \citep{classifierguided, Langevin}.

FP-FM can also be viewed as meta-learning \citep{maml}, with velocity-field approximation as the inner problem and basis-function training as the outer problem. Unlike methods that encode target functions through learned latent representations \citep{conditionalneuralprocesses, neuralstatiscian}, FP-FM computes representations by least squares. Related basis function methods \citep{deepkernel, metalearningwithclosedformsolvers} operate in kernel space and use the Woodbury identity, causing their solves to scale poorly with the amount of data. FP-FM instead operates in feature space and scales efficiently.
The closest work, FE-based neural ODEs \citep{fe-node}, learns basis functions from discrete states of deterministic systems with unobserved derivatives, whereas FP-FM directly fits velocity fields under different data assumptions.

\section{Conclusion}
% What we did
We introduced FP-FM, a flow matching algorithm that enables rapid adaptation to unseen target distributions.
By reframing velocity field estimation in terms of a distribution-weighted inner product, we showed how Hilbert spaces naturally integrate with the flow matching objective.
The three FP-FM variants progressively increase expressivity by varying the coefficient calculation, with more expressive models fitting unseen distributions more accurately. 
% In our experiments, we observe a consistent improvement in precision, recall, and FID as expressivity increases, with \dynamic performing best overall, followed by \temporal and \static.
This opens several directions for future work, such as more fine-grained control of the expressivity.
Overall, FP-FM is a natural first step towards adapting generative models to unseen distributions.

% \paragraph{Broader Impact.}  % NEURIPS
\subsection*{Ethics statement}  % ICLR

FP-FM may lower the barrier to producing highly targeted synthetic data, potentially enabling impersonation, data poisoning, or the amplification of biased or misleading content.
If an industry-grade model is made publicly available, standard security safeguards should be implemented to prevent malicious use. 

\subsection*{AI Use Statement}
In this work, we used generative AI for tasks such as writing code and editing the text. 
We take responsibility for the final content of this work, including text, claims or artifacts produced with the aid of generative AI.

\subsection*{Reproducibility Statement}
Implementation details are described in the appendix. A large set of ImageNet generated samples are attached in the supplementary material. Code will be released upon publication. 

% \begin{ack}
% redacted
% \end{ack}

\bibliography{bib}
\bibliographystyle{iclr2027_conference}

\newpage
\appendix

\section{Derivation of Theorem \ref{thm: sampling}}
\label{app: sampling}

In this section, we show the derivation for Theorem \ref{thm: sampling}, and additionally provide a diagram which makes the intuition clear. See Figure \ref{fig:visual_proof}. In the proof, we omit the $\iota$ superscript for conciseness. 
Note that we make the standard assumptions for this setting: $X_0$ and $X_1$ are independent, $p(x_0)$, $p(x_1)$ exists as densities with respect to the Lebesgue measure, $t<1$, the noise distribution is Gaussian, and the interpolation is linear. 

\begin{proof}
\begin{align}
\mathbb{E}_{X_1,X_0 \mid X_t = x_t}\!\left[X_1 - X_0 | X_t=x_t\right]  &= \int_{\mathcal{X} \times \mathcal{X}} (x_1 - x_0) p(x_1, x_0 | x_t) dx_1 dx_0 \\
    &= \int_{\mathcal{X} \times \mathcal{X}} (x_1 - x_0) \frac{p(x_t | x_1, x_0) p(x_0, x_1)}{p(x_t)} dx_1 dx_0 \\
    &= \int_{\mathcal{X} \times \mathcal{X}} (x_1 - x_0) \frac{p(x_t | x_1, x_0) p(x_0) p(x_1)}{p(x_t)} dx_1 dx_0 \\
    &= \int_{\mathcal{X} \times \mathcal{X}} (x_1 - x_0) \frac{\delta\left(x_t - \left(tx_1 + \left(1-t\right)x_0\right)\right) p(x_0) p(x_1)}{p(x_t)} dx_1 dx_0 
\intertext{
Fix $x_1$ and let $h(x_0) = x_t - \left(tx_1 + \left(1-t\right)x_0\right)$. Then the (only) root $x_0^*=\frac{x_t - tx_1}{1-t}$. Using $\int f(x) \delta\left(h(x)\right) dx = \frac{f(x^*)}{|\text{det}J_h(x^*)|}$, and the fact that $|\text{det}J_h(x_0^*)|=|\text{det}(-(1-t)I)|=|(1-t)^n|$ for our specific choice of $h$,
}
    &= \int_{\mathcal{X}} (x_1 - x^*_0) \frac{ \frac{1}{|\text{det}J_h(x^*_0)|}p(x^*_0) p(x_1)}{p(x_t)} dx_1  \\
    &= \int_{\mathcal{X}} (x_1 - x^*_0) \frac{ \frac{1}{|(1-t)^n|}p(x^*_0) p(x_1)}{p(x_t)} dx_1 \label{eq:17}
\intertext{
Using the same trick as before on $p(x_t)$,
}
    p(x_t) &= \int_{\mathcal{X} \times \mathcal{X}} \delta \left(x_t - \left (tx_1 + \left(1-t\right)x_0\right)\right)p(x_0)p(x_1) dx_1 dx_0 \\
    &= \int_\mathcal{X} \frac{1}{|(1-t)^n|}p(x_0^*)p(x_1)  dx_1.
\intertext{
Then substituting back into \eqref{eq:17},
}
    &= \frac{\int_{\mathcal{X}} (x_1 - x^*_0) \frac{1}{|(1-t)^n|}p(x^*_0) p(x_1) dx_1} {\int_\mathcal{X}  \frac{1}{|(1-t)^n|}p(x_0^*)p(x_1) dx_1} \\
    &= \frac{\int_{\mathcal{X}} (x_1 - x^*_0) p(x^*_0) p(x_1) dx_1} {\int_\mathcal{X}  p(x_0^*)p(x_1) dx_1} \\
    &=E_{X_1}[(X_1 - X^*_0) \frac{ p(X^*_0)}{E_{X_1}[p(X^*_0)]}]
\end{align}
\end{proof}

\paragraph{Sampling Procedure} To approximate this expectation for a given $x^\iota_t$, we first sample a set of $x^\iota_1$'s. Typically, this consists of the provided samples from the new distribution. Then, for each $x_1^{\iota, i}$, compute the corresponding initial state $x_0^{\iota, i}$ and then $w^{\iota,i}=p(x_0^{\iota, i})$. Normalize the $w^{\iota,i}$'s to account for $E_{X_1}[p(X^*_0)]$ with $\Tilde{w}^{\iota,j} = w^{\iota,j} / \sum_{i=1}^m w^{\iota,i}$. Lastly, compute the empirical mean weighted by $\Tilde{w}^{\iota,i}$, 
\begin{equation}
    \mathbb{E}_{X_1,X_0 \mid X_t = x}\!\left[X_1 - X_0 | X_t=x_t\right]  \approx \sum_{i=1}^m (x_1^{\iota, i} - x_0^{\iota, i}) \Tilde{w}^{\iota,i}
\end{equation}

\begin{figure}
    \centering
    \includegraphics[width=0.6\linewidth]{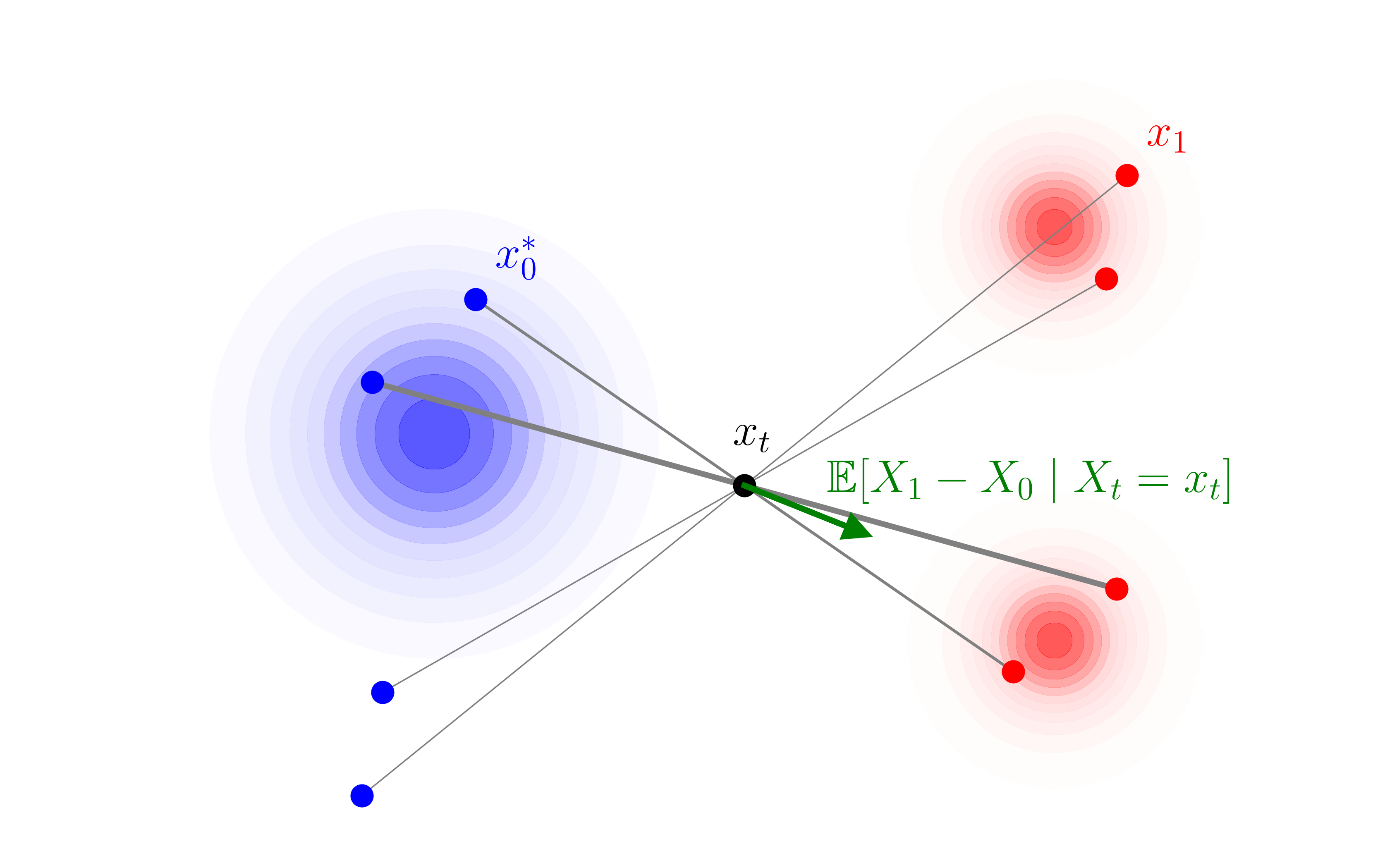}
    \caption{\textbf{Visualization of Theorem \ref{thm: sampling}.} Consider a noise distribution (shown in blue) and a mixture of Gaussians target distribution (shown in red). For any given $x_t$ (black point) and set of $x_1$'s (red points), we may compute the corresponding $x_0^*$'s using the equation $x^*_0 =\frac{x_t - tx_1}{1-t} $. However, not all $x_0^*$'s are probable. We represent the likelihood of $x_0^*, x_1$ given $x_t$ via the thickness of the line connecting them. Clearly, some pairs are more likely. Therefore, the expected velocity field $\mathbb{E}[X_1-X_0|X_t]$ (green) will appear to align with the thicker of these lines. This is exactly why we need to consider the probability of $x_0^*$ when computing the expected velocity field.   }
    \label{fig:visual_proof}
\end{figure}

\section{Function-Distribution Pairs} \label{app:function-distribution-pairs}

We begin by generalizing this formulation. Let $\mathcal{F}=\{f: \mathcal{X} \to \mathbb{R}^n\ | f \text{ measurable}\}$, where compact $\mathcal{X}$ is the input space and $\mathbb{R}^n$ is the output space. Suppose we observe a finite subset $\mathcal{T} = \{f^i\}_{i=1}^m \subset \mathcal{F}$,  and a corresponding set of densities $\mathcal{P}=\{p^i_X: \mathcal{X} \to [0,\infty)| i \in 1...m\}$. 
Each density $p^i_X$ induces a Hilbert space $\mathcal{H}^i=L^2(p_X^i)$ with inner product $\langle f, g\rangle_i= \int_\mathcal{X} f(x)^\top g(x) p_X^i(x)dx$, and assume $f^i \in \mathcal{H}^i.$
Note that for $j \neq i$, it is possible that $f^j \notin \mathcal{H}^i$. In other words, our set of functions may or may not exist in each other's function spaces.

As is typical for the function encoder, our goal is to learn a set of functions $\{g^j\}_{j=1}^k$ such that for all $i$, $f^i \approx \sum_{j=1}^k c^{i,j} g^j$ where $c^i \in \mathbb{R}^k$. This allows us to approximate $f^i$ with a linear combination of our basis functions. As is typical, we compute the coefficients using least squares,

\begin{equation} \label{eqn: app-ls}
 \mathbb{R}^k \ni c^i = \begin{bmatrix}
\langle g^1, g^1 \rangle_{i} & \hdots & \langle g^1, g^k \rangle_{i} \\
\vdots & \ddots & \vdots \\
\langle g^k, g^1 \rangle_{i} & \hdots & \langle g^k, g^k \rangle_{i} \\
\end{bmatrix}^{-1}
\begin{bmatrix}
\langle f^i, g^1 \rangle_{i} \\
\vdots \\
\langle f^i, g^k \rangle_{i} \\
\end{bmatrix}.
\end{equation}

This requires $\forall i,j       g^j \in \mathcal{H}^i$, i.e., the basis functions must be members of all Hilbert spaces. However, since we are approximating the basis functions with neural networks, they are continuous. Since $\mathcal{X}$ is compact and $g^j$ is continuous, then it follows from the extreme value theorem that each $g^j$ is bounded by some maximum value $M_j$. Therefore,

\[ ||g^j||^2_{L^2(p_X^i)} = \int_\mathcal{X} ||g^j(x)||_2^2 p_X^i(x)dx  \leq M_j^2\int_X p_X^i(x)dx = M_j^2<  \infty\]

Therefore, $g^j \in \mathcal{H}^i$. In other words, the basis functions are well-defined regardless of the density function. 
Thus, we only need to train the basis functions to span $\{f^i\}$, and then we may use them to approximate any such function $f^i$. Note that by using Equation \eqref{eqn: app-ls}, we implicitly assume we want to minimize the approximation error on points that are likely under $p_X^i$, but assumption is quite natural given the correspondence between $f^i$ and $p^i_X$. 

One interesting result of this formulation is that two functions $f^i, f^j$ may have the same representation, $c^i = c^j$, even if they differ on a set with measure greater than 0. For example, suppose the support of $p_X^i$ and $p_X^j$ are disjoint; Then it is easy to see how they may have the same coefficient representation. This implies we use the same approximation for both functions because this approximation is the most accurate when measured via the corresponding inner product. 
In terms of flow matching, if the representation for two target distributions were the same, then they would have the same velocity field, and therefore the same sampling distribution. However, since every $X_t^\iota$ has the same noise distribution, the supports of the distributions $p_{X_t^\iota}^\iota$ overlap when $t=0$, which helps to distinguish between the velocity fields. We have not observed two separate velocity fields having the same coefficients in practice.

\section{Baselines and Implementation Details} \label{app:baselines}

\begin{figure}
    \centering
    \includegraphics[width=0.6\linewidth,trim=50pt 0pt 20pt 0pt, clip]{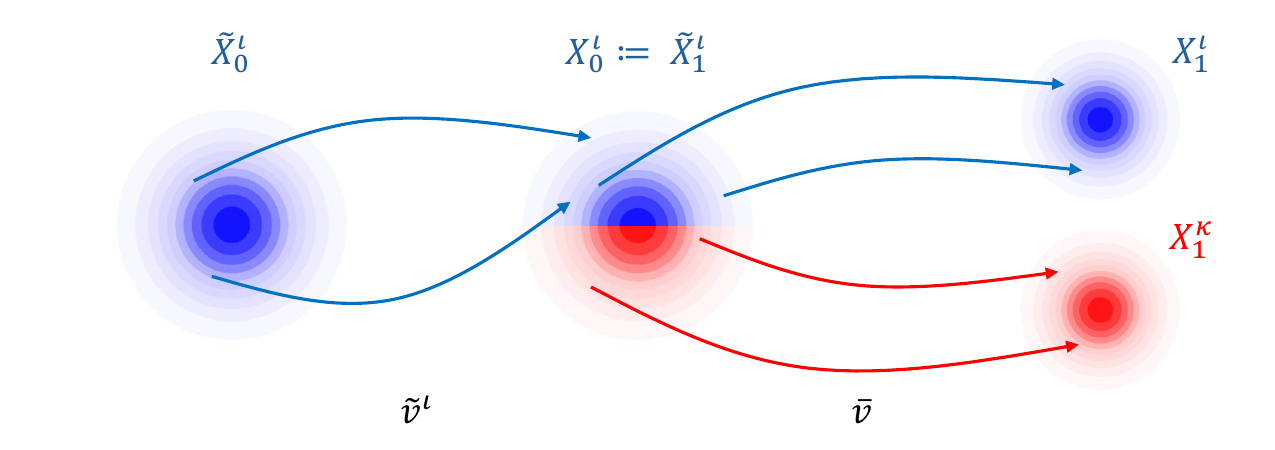}
    \caption{\textbf{Visualization of the Distribution-Guided Model.} Suppose we have trained a \texttt{Unconditional} model  $\Bar{v}$ that maps a Normal distribution (middle) to a mixture of two Gaussians (right). The \texttt{Unconditional} model maps the top half of its noise distribution to the upper target Gaussian, and the bottom half to the lower target Gaussian. Then, we are interested in generating samples only from the upper Gaussian distribution (i.e., $X_1^\iota$). To do so, train a second velocity field $\Tilde{v}^\iota$ which itself maps from a Normal distribution (left) to the upper half of the Normal distribution (middle, blue). Putting these together, a sample is first transitioned from $\Tilde{X}_0^\iota$ to $\Tilde{X}_1^\iota$ via $\Tilde{v}^\iota$, then transitioned from $X_0^\iota := \Tilde{X}_1^\iota$ to $X_0^\iota$ via $\Bar{v}$. This process yields samples from the target distribution.} 
    \label{fig:distribution-guided-vis}
\end{figure}

\paragraph{Unconditional} is a standard flow matching model trained on all of the training distributions as if they were a single distribution (e.g., a mixture). This acts as naive baseline, and as a starting point for some of the other algorithms. For obvious reasons, this baseline should perform poorly at all of TD, UD, US. 

\paragraph{Conditional} is a model provided with a conditioning variable to describe the target distribution, but is otherwise the same as an \texttt{Unconditional} model. For example, for \textit{2D-Arcs}, the conditioning variable is the angle of the center of the arc. While this allows the model to accurate learn the training distributions, it is unclear how this conditioning variable could generalize to a mixture of arcs, or a spiral. For \textit{MNIST}, the conditioning variable is a 10-dimensional vector with 1's for classes within the mixture. Although this conditioning variable can describe mixtures (with multiple ones), it was not trained to do so and so it fails to generalize. Likewise, although it has a one-hot encoding of 9, it has not seen this during training and so it fails to generalize. \textit{ImageNet} uses a similar strategy as MNIST, with classes encoded into a vector. We expect this model to perform well on TD's, but to fail on UD's and US's. 

\paragraph{Classifier-Guided} starts with a pre-trained, \texttt{Unconditional} model $\Bar{v}$. Then, it trains a classifier $\hat{p^\iota}: \mathcal{X} \times [0,1] \to [0,1]$ to label if a given state-time pair $(x_t,t)$ belongs to the target distribution or not. The score of this classifier is used to guide the \texttt{Unconditional} model towards the desired distribution,
\[ v^\iota(x_t,t) \approx \Bar{v}(x_t,t) + \alpha \nabla_{x_t} \text{log}\hat{p}^\iota(x_t,t), \]
where $\alpha$ is a hyper parameter. 
The intuition behind this idea is that the \texttt{Unconditional} model can generate samples from any of the training distributions, and so we only need to guide the model towards the target distribution using the gradient of the log density as estimated by our model. The main challenge with this method is choosing $\alpha$. To the best of our knowledge, there is not a consistent way of choosing $\alpha$, and it needs to be tuned per dataset and potentially even per distribution. Also, this requires training a classifier per each target distribution, which is similar in expense to finetuning. 
We expect this model to perform well on TD and UD. However, it will typically perform poorly on US.

\paragraph{Distribution-Guided} likewise starts with a pretrained \texttt{Unconditional} model $\Bar{v}$. It then trains a secondary model $\Tilde{v}^\iota$ to produce noise inputs that, when passed through the pretrained model, generate samples from the target distribution.
Generating a sample has the following form:
\[\Tilde{X}^\iota_0 \sim \mathcal{N}(0, I)\]
\[ \Tilde{X}^\iota_t = \Tilde{X}_0^\iota + \int_0^t \Tilde{v}^\iota (\Tilde{X}_\tau^\iota) d\tau \]
\[ X_0^\iota := \Tilde{X}_1^\iota\]
\[ {X}^\iota_t = {X}_0^\iota + \int_0^t \Bar{v} ({X}_\tau^\iota) d\tau \]

Similar to \texttt{Classifier-Guided}, the idea is that the \texttt{Unconditional} model can already generate any sample in the training distributions. Thus, we only need to figure out which distribution of inputs to the \texttt{Unconditional} model will yield the target distribution. Then, we train a separate model to generate this starting distribution. See Figure \ref{fig:distribution-guided-vis}.

While this approach is mathematically elegant, it has numerous practical downsides. First, to train the distribution-specific model, we must first compute the inputs to the \texttt{Unconditional} model that generate the target distribution. This involves starting with samples from the target distribution, and integrating backwards in time to recover the initial states. This backwards integration can be numerically unstable. Second, we must train an entirely separate flow matching model, which is expensive. Third, generation is now twice as expensive, since we have two models instead of one.
We expect this model to succeed on TD and UD. It should fail on US, since it can only generate samples that the base \texttt{Unconditional} model is capable of generating.

\paragraph{Finetune} is the standard approach. Given a pretrained unconditional model, finetune it on the data provided for the target distribution using the standard flow matching method. This approach can in theory perform well at all of TD, UD, US. However, since we are training then entire model, it may require many samples and a lot of compute to do so. 

\paragraph{Implementation Details. } 

% \paragraph{Implementation Details.} 
Across all methods, we align training settings as closely as possible, including architectures, batch sizes, and optimizers.
We use an Adam or AdamW optimizer with a learning rate of 1e-3 or 5e-4 and constant or cosine decay learning rate schedule. 
The architecture is dataset dependent: MLP for 2D Arcs, UNet \citep{unet} for MNIST, and latent DiT (B/2) \citep{vit, stablediffusion, Peebles2022DiT} for ImageNet.  For ImageNet, we use a pretrained VAE latent space. The FP-FM baselines all use a single neural network body with multiple output heads \citep{fe}. Thus, the cost of the basis functions in terms of parameter count is relatively minor. For example, on ImageNet, the FP-FM baselines use only $0.62\%$ more parameters than the conditional model. For all models, we integrate to $1-\epsilon$ for $\epsilon=1e-3$ to avoid numerical instability. 

We used 1,000 samples per task for Arc and 70 for MNIST. All FP-FM variants used 99 basis functions. Training batches contained 10 task distributions for Arc and 9 for MNIST, corresponding to 10,000 points and 630 images per optimization step, respectively. Models were trained for 10,000 gradient steps on Arc and 20,000 on MNIST. Samples were generated by integrating the learned vector field with a fixed-step fourth-order Runge–Kutta solver over 50 intervals.

All ImageNet models use the same DiT-B/2 backbone \citep{Peebles2022DiT}, operating on $4\times32\times32$ latent representations of $256\times256$ images produced by a frozen Stable Diffusion VAE (sd-vae-ft-ema, latent scale 0.18215). The conditional baseline is the standard DiT-B/2 with class embeddings. FP-FM uses the same backbone but removes class conditioning entirely. Its only architectural modification is a wider output layer that predicts the marginal field $G_0$ together with $K$ basis fields, increasing the parameter count by less than 1\%.
We train all models from scratch with the flow-matching objective on 900 ImageNet classes. The remaining 100 classes are held out during training and are not seen by any method, including the baselines. At inference, each FP-FM variant is given $m=128$ context images from the target class. The corresponding coefficients are obtained by ridge-regularized least-squares projection onto the learned basis. This projection is performed once per class for Static FP-FM, at each time step for Temporal FP-FM, and at each state and time step for Dynamic FP-FM.
We generate samples by integrating the resulting velocity field with a 50-step RK4 solver over $t\in[0,1-10^{-3}]$. Classifier-free-style guidance is applied only to the class-dependent component of the field. For the fine-tuning baseline, we add rank-8 LoRA adapters to the attention and MLP projections of the conditional DiT and optimize them using the same 128 context images per class.
For evaluation, each class provides 256 real images, divided into 128 context images and a disjoint set of 128 reference images. Every method generates 128 samples per class using the same set of 128 initial noise vectors. We compute FID, precision, and recall separately for each class using 128 generated and 128 reference images, and report the average over the 100 classes in each split.

Coefficients were obtained using ridge-regularized least squares, \(c=(G^\top G+\lambda I)^{-1}G^\top f\). We used \(\lambda=10^{-3}\). 
In addition, we controlled basis scaling using \(\lVert\operatorname{diag}(G^\top G)-1\rVert_2^2\) as a regularizer during training \citep{fe}. This penalty encourages unit-norm basis functions and prevents the scale ambiguity of the least-squares objective from causing unbounded basis magnitudes.

2D Arcs and MNIST experiments are run on a desktop machine with a 3080. The ImageNet experiments are run on a cluster of 8 RTX PRO 6000 Blackwell GPUs. 
2D Arcs experiments run in a few minutes, depending on the algorithm. MNIST experiments take a few hours. 
ImageNet experiments take about a day per algorithm.

For all approaches involving additional training, we train for 1000 gradient steps using the Adam optimizer with a learning rate of $1e-3$. For \dg, we use the same architecture as the \unconditional model. For \cg, we use the same architecture as the \unconditional model except slightly modified to output probabilities. We train the classifier using cross entropy with positive samples from $X_t^\iota$ and negative samples from $X_t^\kappa$ for $\kappa \neq \iota$. \finetune uses full finetuning for 2D Arcs and MNIST, and LoRA for ImageNet due to memory constraints. 

ImageNet is significantly more challenging than the other datasets due to the high-dimensional and diverse nature of its images. 
We model each channel of the latent (which is a WxHxC image)  as a separate function space, which slightly increases the expressivity of the coefficients. 
Note that $W\times H > k$.
We also use one additional function encoder technique for ImageNet, which improves performance:

\paragraph{Guided FP-FM}
As introduced in \citep{fe}, the \textit{residuals method} of function encoders represents a function as 
$f = \Bar{f} + \sum_{i=1}^k c^ig^i,$
where $\Bar{f}$ is a neural net trained to model the average function in the training set. Therefore, the basis functions correct the error of mean function. We may leverage a similar technique in the flow matching setting, where $\Bar{f}$ fits the mean velocity field, and the basis functions \textit{guide} this velocity field towards a specific distribution, similar to classifier-free guidance.

\newpage
\section{Full Results} \label{app:fullresults}
\begin{table*}[h!]
\centering
\small
\setlength{\tabcolsep}{4pt}
\begin{tabular}{l|rrr|rrr|r}
\toprule
& \multicolumn{3}{c|}{Precision $\uparrow$} & \multicolumn{3}{c|}{Recall $\uparrow$} & Time (s) $\downarrow$ \\
Algorithm & TD & UD & US & TD & UD & US & \\
\midrule
\unconditional & 0.108 & 0.291 & 0.003 & 0.997 & 0.999 & 0.032 & 0.17 \\
 & \scriptsize $\pm$ 0.013 & \scriptsize $\pm$ 0.036 & \scriptsize $\pm$ 0.001 & \scriptsize $\pm$ 0.002 & \scriptsize $\pm$ 0.001 & \scriptsize $\pm$ 0.005 & \scriptsize $\pm$ 0.01 \\
\conditional & 0.369 & 0.608 & 0.006 & 0.998 & 0.613 & 0.041 & 0.17 \\
 & \scriptsize $\pm$ 0.049 & \scriptsize $\pm$ 0.045 & \scriptsize $\pm$ 0.001 & \scriptsize $\pm$ 0.002 & \scriptsize $\pm$ 0.026 & \scriptsize $\pm$ 0.006 & \scriptsize $\pm$ 0.01 \\
\cg & 0.407 & 0.654 & 0.009 & 0.836 & 0.867 & 0.134 & 7.08 \\
 & \scriptsize $\pm$ 0.040 & \scriptsize $\pm$ 0.081 & \scriptsize $\pm$ 0.001 & \scriptsize $\pm$ 0.016 & \scriptsize $\pm$ 0.019 & \scriptsize $\pm$ 0.068 & \scriptsize $\pm$ 0.22 \\
\dg & 0.980 & 0.980 & 0.016 & 0.999 & 0.999 & 0.997 & 16.48 \\
 & \scriptsize $\pm$ 0.002 & \scriptsize $\pm$ 0.003 & \scriptsize $\pm$ 0.002 & \scriptsize $\pm$ 0.000 & \scriptsize $\pm$ 0.000 & \scriptsize $\pm$ 0.002 & \scriptsize $\pm$ 0.28 \\
\finetune & 0.612 & 0.735 & 0.385 & 0.998 & 0.999 & 0.999 & 4.77 \\
 & \scriptsize $\pm$ 0.033 & \scriptsize $\pm$ 0.018 & \scriptsize $\pm$ 0.019 & \scriptsize $\pm$ 0.001 & \scriptsize $\pm$ 0.001 & \scriptsize $\pm$ 0.001 & \scriptsize $\pm$ 0.12 \\
\static & 0.324 & 0.320 & 0.086 & 0.847 & 0.811 & 0.894 & 0.21 \\
 & \scriptsize $\pm$ 0.066 & \scriptsize $\pm$ 0.081 & \scriptsize $\pm$ 0.023 & \scriptsize $\pm$ 0.107 & \scriptsize $\pm$ 0.052 & \scriptsize $\pm$ 0.043 & \scriptsize $\pm$ 0.01 \\
\temporal & 0.586 & 0.763 & 0.709 & 0.997 & 0.997 & 0.998 & 0.72 \\
 & \scriptsize $\pm$ 0.043 & \scriptsize $\pm$ 0.025 & \scriptsize $\pm$ 0.012 & \scriptsize $\pm$ 0.002 & \scriptsize $\pm$ 0.001 & \scriptsize $\pm$ 0.001 & \scriptsize $\pm$ 0.00 \\
\dynamic & 0.991 & 0.997 & 0.934 & 0.989 & 0.985 & 0.972 & 10.99 \\
 & \scriptsize $\pm$ 0.001 & \scriptsize $\pm$ 0.001 & \scriptsize $\pm$ 0.002 & \scriptsize $\pm$ 0.001 & \scriptsize $\pm$ 0.002 & \scriptsize $\pm$ 0.002 & \scriptsize $\pm$ 0.04 \\
\reference & 0.999 & 0.999 & 0.999 & 0.999 & 0.999 & 0.999 & - \\
 & \scriptsize $\pm$ 0.001 & \scriptsize $\pm$ 0.000 & \scriptsize $\pm$ 0.001 & \scriptsize $\pm$ 0.000 & \scriptsize $\pm$ 0.001 & \scriptsize $\pm$ 0.000 & - \\
\bottomrule
\end{tabular}
\caption{\textbf{2D Arc.} Comparison of methods across TD, UD, and US settings. Arrows indicate whether higher ($\uparrow$) or lower ($\downarrow$) is better. Results are averaged over 5 seeds. Standard deviation is shown after the $\pm$ symbol.}
\label{tab:arc_quant_table}
\end{table*}
\begin{table*}[h!]
\centering
\tiny
\setlength{\tabcolsep}{4pt}
\begin{tabular}{l|rrr|rrr|rrr|r}
\toprule
& \multicolumn{3}{c|}{Precision $\uparrow$} & \multicolumn{3}{c|}{Recall $\uparrow$} & \multicolumn{3}{c|}{FID $\downarrow$} & Time (s) $\downarrow$ \\
Algorithm & TD & UD & US & TD & UD & US & TD & UD & US & \\
\midrule
\unconditional & 0.342 & 0.464 & 0.303 & 0.705 & 0.705 & 0.666 & 97.93 & 88.59 & 92.84 & 10.31 \\
 & \scriptsize $\pm$ 0.172 & \scriptsize $\pm$ 0.233 & \scriptsize $\pm$ 0.165 & \scriptsize $\pm$ 0.353 & \scriptsize $\pm$ 0.353 & \scriptsize $\pm$ 0.337 & \scriptsize $\pm$ 123.96 & \scriptsize $\pm$ 126.66 & \scriptsize $\pm$ 134.69 & \scriptsize $\pm$ 0.29 \\
\conditional & 0.875 & 0.763 & 0.309 & 0.877 & 0.759 & 0.783 & 12.54 & 23.18 & 31.92 & 10.78 \\
 & \scriptsize $\pm$ 0.011 & \scriptsize $\pm$ 0.077 & \scriptsize $\pm$ 0.059 & \scriptsize $\pm$ 0.013 & \scriptsize $\pm$ 0.048 & \scriptsize $\pm$ 0.074 & \scriptsize $\pm$ 0.91 & \scriptsize $\pm$ 5.10 & \scriptsize $\pm$ 3.19 & \scriptsize $\pm$ 0.19 \\
\cg & 0.582 & 0.632 & 0.440 & 0.719 & 0.720 & 0.714 & 82.00 & 80.18 & 88.98 & 29.19 \\
 & \scriptsize $\pm$ 0.292 & \scriptsize $\pm$ 0.317 & \scriptsize $\pm$ 0.227 & \scriptsize $\pm$ 0.360 & \scriptsize $\pm$ 0.360 & \scriptsize $\pm$ 0.360 & \scriptsize $\pm$ 131.12 & \scriptsize $\pm$ 127.91 & \scriptsize $\pm$ 137.06 & \scriptsize $\pm$ 0.32 \\
\dg & 0.753 & 0.725 & 0.686 & 0.781 & 0.760 & 0.734 & 18.39 & 21.99 & 15.71 & 626.17 \\
 & \scriptsize $\pm$ 0.040 & \scriptsize $\pm$ 0.051 & \scriptsize $\pm$ 0.119 & \scriptsize $\pm$ 0.073 & \scriptsize $\pm$ 0.157 & \scriptsize $\pm$ 0.117 & \scriptsize $\pm$ 1.44 & \scriptsize $\pm$ 4.35 & \scriptsize $\pm$ 2.00 & \scriptsize $\pm$ 2.19 \\
\finetune & 0.728 & 0.728 & 0.737 & 0.590 & 0.606 & 0.566 & 80.31 & 81.37 & 81.45 & 215.71 \\
 & \scriptsize $\pm$ 0.364 & \scriptsize $\pm$ 0.364 & \scriptsize $\pm$ 0.370 & \scriptsize $\pm$ 0.297 & \scriptsize $\pm$ 0.303 & \scriptsize $\pm$ 0.285 & \scriptsize $\pm$ 132.90 & \scriptsize $\pm$ 129.52 & \scriptsize $\pm$ 139.73 & \scriptsize $\pm$ 0.16 \\
\static & 0.783 & 0.757 & 0.554 & 0.865 & 0.856 & 0.766 & 16.71 & 21.03 & 22.53 & 10.63 \\
 & \scriptsize $\pm$ 0.027 & \scriptsize $\pm$ 0.023 & \scriptsize $\pm$ 0.100 & \scriptsize $\pm$ 0.017 & \scriptsize $\pm$ 0.027 & \scriptsize $\pm$ 0.090 & \scriptsize $\pm$ 1.78 & \scriptsize $\pm$ 1.04 & \scriptsize $\pm$ 8.14 & \scriptsize $\pm$ 0.03 \\
\temporal & 0.771 & 0.767 & 0.569 & 0.885 & 0.897 & 0.797 & 14.13 & 17.49 & 17.69 & 20.78 \\
 & \scriptsize $\pm$ 0.028 & \scriptsize $\pm$ 0.024 & \scriptsize $\pm$ 0.115 & \scriptsize $\pm$ 0.018 & \scriptsize $\pm$ 0.014 & \scriptsize $\pm$ 0.029 & \scriptsize $\pm$ 0.52 & \scriptsize $\pm$ 0.37 & \scriptsize $\pm$ 1.64 & \scriptsize $\pm$ 0.01 \\
\dynamic & 0.759 & 0.829 & 0.723 & 0.818 & 0.817 & 0.837 & 17.38 & 18.38 & 18.37 & 21.96 \\
 & \scriptsize $\pm$ 0.081 & \scriptsize $\pm$ 0.061 & \scriptsize $\pm$ 0.114 & \scriptsize $\pm$ 0.032 & \scriptsize $\pm$ 0.022 & \scriptsize $\pm$ 0.078 & \scriptsize $\pm$ 4.03 & \scriptsize $\pm$ 3.22 & \scriptsize $\pm$ 7.02 & \scriptsize $\pm$ 0.01 \\
\reference & 0.891 & 0.898 & 0.900 & 0.897 & 0.893 & 0.903 & 11.16 & 12.91 & 10.27 & - \\
 & \scriptsize $\pm$ 0.015 & \scriptsize $\pm$ 0.012 & \scriptsize $\pm$ 0.066 & \scriptsize $\pm$ 0.009 & \scriptsize $\pm$ 0.010 & \scriptsize $\pm$ 0.019 & \scriptsize $\pm$ 0.35 & \scriptsize $\pm$ 0.37 & \scriptsize $\pm$ 0.32 & - \\
\bottomrule
\end{tabular}
\caption{\textbf{MNIST.} Comparison of methods across TD, UD, and US settings. Arrows indicate whether higher ($\uparrow$) or lower ($\downarrow$) is better. Results are averaged over 5 seeds. Standard deviation is shown after the $\pm$ symbol.}
\label{tab:mnist_quant_table}
\end{table*}

\clearpage

\section{Ablations} \label{app:ablation}

\subsection{Number of Shots}

\begin{figure}[h]
    \centering
    \includegraphics[width=1.0\linewidth]{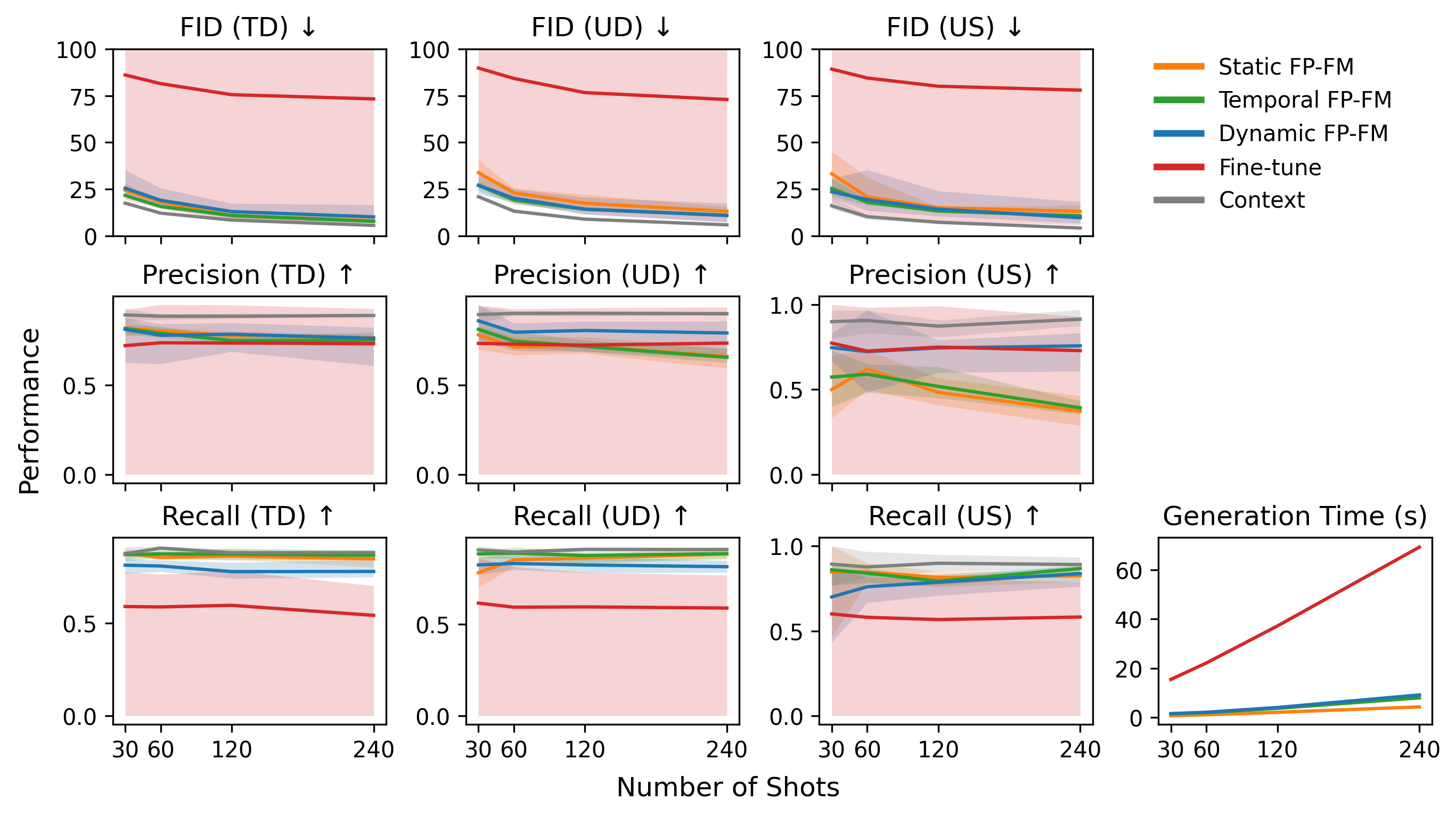}
    \caption{\textbf{Ablation on the Number of Shots.} We vary the number of shots both during training and evaluation for all three FP-FM baselines over 5 seeds. Plots show the mean, min, and max for each algorithm. 
    }
    \label{fig:ablate_shots}
\end{figure}

We vary the number of shots provided during training and evaluation on the MNIST dataset. We observe that for all three FP-FM variants, FID decreases as the number of shots increases. 
\dynamic has stable precision and recall as the number of shots increases. For \static and \temporal, precision may slightly decrease as the number of shots increases. 
\reference provides an approximate upper bound on performance as it compares the ground truth provided context to the real images.  

\subsection{Number of Basis Functions}

\begin{figure}[h]
    \centering
    \includegraphics[width=1.0\linewidth]{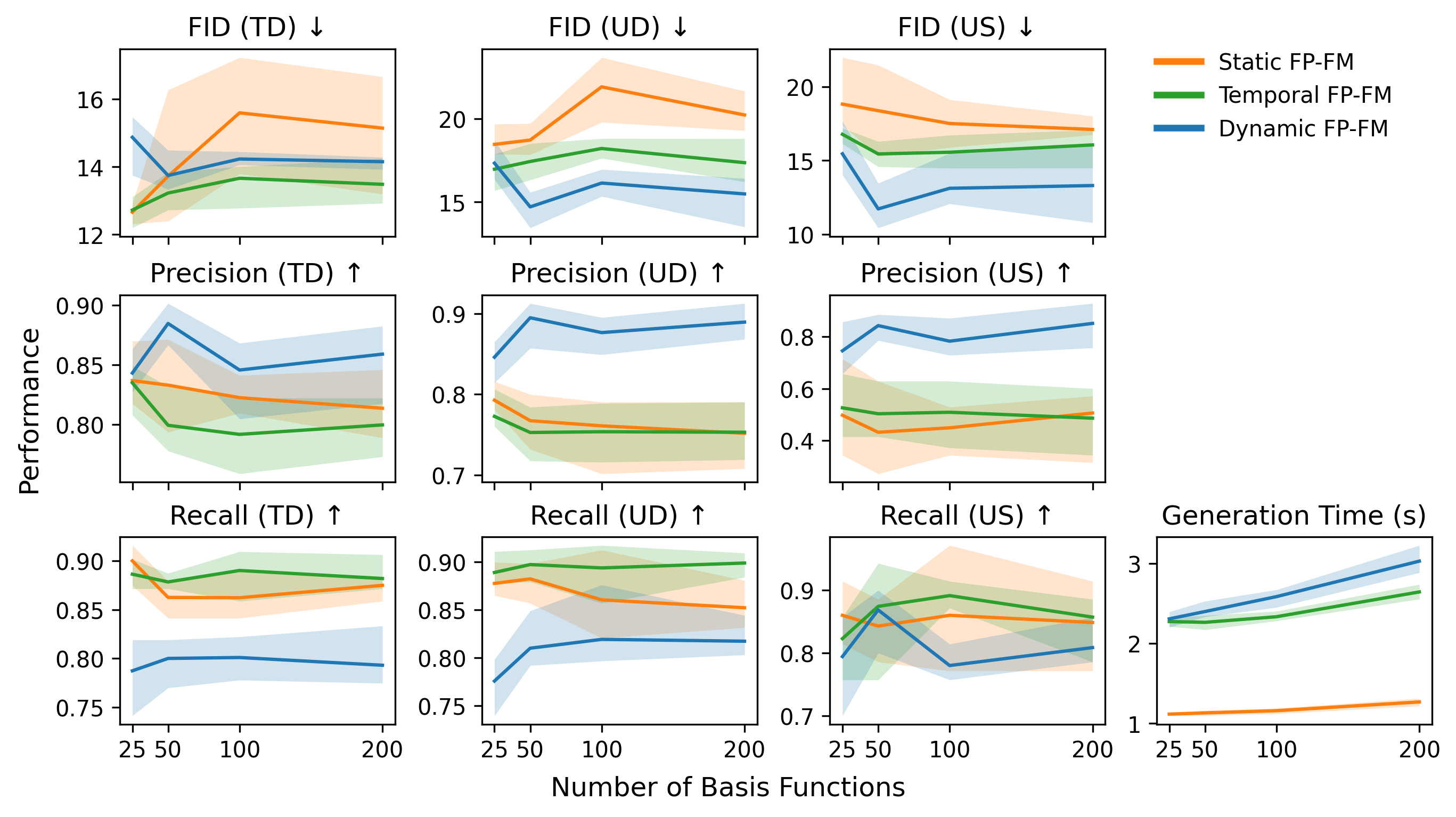}
     \caption{\textbf{Ablation on the Number of Basis Functions.} We vary the number of basis functions for all three FP-FM baselines on the MNIST dataset.}
    \label{fig:ablate_basis}
\end{figure}

We likewise ablate the number of basis functions on the MNIST dataset.
We find that none of \static, \temporal, or \dynamic are very sensitive to the number of basis functions on this dataset after a certain threshold is reached. In other words, adding more basis functions does not degrade performance. 
Generation time slightly increases as the number of basis functions increases because the cost of the least-squares solve increases. 
Since \dynamic requires the most least-squares solves, its generation time is the most affected of all three FP-FM variants, though the increase is relatively small.

\subsection{Dynamic FP-FM Overfitting Analysis}
Dynamic FP-FM is susceptible to overfitting due to the importance sampling estimator. However, the basis functions play a crucial role in preventing this overfitting. To verify this, we introduce a nearest-neighbor statistic $\tau$: the median distance from a generated sample to the calibration set, divided by the median distance from a held-out real image of the same class to that same set. A value of $\tau=1$ means generations sit exactly as far from the provided examples as a fresh photograph does; $\tau << 1$ indicates copying. We also report the fraction of generations falling below the 5th percentile of the real-image distances, for which 0.05 is expected by construction.

We find that the nearest neighbor statistic $\tau = 1.035$, and the fraction of generations falling below the 5th percentile of real distances is 0.075. So, Dynamic FP-FM is only slightly overfitting.

\section{Extended Discussion and FAQ}

The questions below address the main technical and experimental concerns that might arise when interpreting our results. We explain when the basis projection is meaningful, why it regularizes the empirical velocity estimate, and why its added generation cost is small. We also clarify the parameter comparison, baseline choices, and number of basis functions, while stating the limits of the method in few-shot settings. 

\paragraph{Why is Equation \eqref{eqn: stfe-ls} not singular and/or trivial?}

Equation \eqref{eqn: stfe-ls} projects the $n$-dimensional estimated velocity onto a set of $k$ $n$-dimensional basis functions. 
As argued in prior work \citep{fe}, the learned basis functions are independent at initialization, and the training loss discourages dependence, so it is typically the case that the basis is independent. Therefore, we can view this as finding the nearest vector within the $k$-dimensional subspace of $\mathbb{R}^n$. When $k<n$, this projection indeed regularizes the estimator by forcing the result to lie within the span of the basis. For most image-based problems, $k<n$. If $k\geq n$, then the projection becomes trivial as the basis spans all of $\mathbb{R}^n$. In particular, this occurs on the toy example 2D Arcs.

\paragraph{Does this work in few-shot settings?} Specifying a data manifold for high-dimensional manifolds such as images from only a few data points is ill-posed. So, the method is still applicable, but we expect the results to be quite poor because there simply is not enough information to estimate the manifold. We believe no method that relies on example data alone should perform well in such a setting without making heavy use of priors on what the distributions should look like. Incorporating such priors to target the few-shot setting is beyond the scope of this work. 

\paragraph{Isn't the computational complexity of \dynamic quite high during generation? }

The main expense for generating data is forward passing the neural network. Any data processing, such as computing the empirical velocity estimate, is extremely cheap. Both \dynamic and other methods such as \finetune require one model forward pass per generation step. \dynamic additionally requires data processing and a $k$-dimensional least squares calculation. Both of these are negligible in comparison to the forward pass, and hence the marginal cost of generating data is only slightly higher than \conditional or \finetune. 

\paragraph{Weight collapse in Theorem \ref{thm: sampling}?}
In high-dimensional data spaces $\mathbb{R}^n$, the weights assigned to each sample tend to concentrate around a single sample. The basis functions play the crucial role of preventing generation from always following a single sample. Because we project this vector onto the basis, the output cannot directly overfit to empirical estimate, and the velocity field instead points in a direction that is generally useful for generation as prescribed by the learned basis. Consequently, one interpretation of \dynamic is a learned, regularized version of a purely data-driven empirical estimate. 

\paragraph{Why does the \conditional baseline  use a one-hot embedding?}
One-hot embeddings are the standard approach for conditioning on a discrete set of values, in this case the target class. One may wonder why we do not use an embedding table, similar to a language model. In fact, these two methods are mathematically equivalent. Consider multiplying a matrix $W$ by the one-hot vector $v$ with $v_i=1$ and $v_j = 0$ for $j\neq i$. The output will be the $i$'th column of $W$, which is exactly the same as simply "looking up" the $i$'th column using index $i$ rather than a vector $v$. 
Therefore, embedding tables are effectively sparse implementations of one-hot embeddings.

\paragraph{Doesn't FP-FM require $k$-times as many parameters as a \conditional model? Is the difference in performance due to this extra capacity?}

As is typical for function encoders \citep{fe}, we use a single body with multiple output heads to parameterize the basis. Effectively, this means FP-FM shares almost entirely the same architecture as \conditional, except in the last few layers. Therefore, the difference in parameters is relatively small and does not explain the difference in performance. In particular, the total number of parameters is only 0.62\% larger for \dynamic relative to \conditional on ImageNet. 

 \paragraph{How do you choose the number of basis functions?}
 As shown in prior work \citep{fe}, using too many basis functions does not negatively impact performance. Therefore, it is best to err on the side of caution and use more basis functions than needed. 
 To that end, we use 99 basis functions for 2D Arcs and MNIST, and 512 for ImageNet. 
 Prior work has also investigated how to determine the minimum number of basis functions exactly, see \citet{fe_min}.

\begin{table*}[b]
\centering
\small
\setlength{\tabcolsep}{4pt}
\begin{tabular}{l|rrr|rrr|rrr|r}
\toprule
& \multicolumn{3}{c|}{Precision $\uparrow$} & \multicolumn{3}{c|}{Recall $\uparrow$} & \multicolumn{3}{c|}{FID $\downarrow$} & Time (s) $\downarrow$ \\
Algorithm & TD & UD & US & TD & UD & US & TD & UD & US & \\
\midrule
\loo & 0.792 & 0.791 & 0.734 & 0.640 & 0.702 & 0.637 & 18.01 & 20.03 & 17.85 & 21.58 \\
 & \scriptsize $\pm$ 0.017 & \scriptsize $\pm$ 0.024 & \scriptsize $\pm$ 0.084 & \scriptsize $\pm$ 0.012 & \scriptsize $\pm$ 0.038 & \scriptsize $\pm$ 0.106 & \scriptsize $\pm$ 1.03 & \scriptsize $\pm$ 1.31 & \scriptsize $\pm$ 1.61 & \scriptsize $\pm$ 0.06 \\
\dynamic & 0.759 & 0.829 & 0.723 & 0.818 & 0.817 & 0.837 & 17.38 & 18.38 & 18.37 & 21.96 \\
 & \scriptsize $\pm$ 0.081 & \scriptsize $\pm$ 0.061 & \scriptsize $\pm$ 0.114 & \scriptsize $\pm$ 0.032 & \scriptsize $\pm$ 0.022 & \scriptsize $\pm$ 0.078 & \scriptsize $\pm$ 4.03 & \scriptsize $\pm$ 3.22 & \scriptsize $\pm$ 7.02 & \scriptsize $\pm$ 0.01 \\
\bottomrule
\end{tabular}
\caption{\textbf{Leave-one-out ablation on MNIST.} Comparison of \loo relative to \dynamic across TD, UD, and US settings. Arrows indicate whether higher ($\uparrow$) or lower ($\downarrow$) is better. Results are averaged over 5 seeds. Standard deviation is shown after the $\pm$ symbol.}
\label{tab:mnist_quant_table_loo}
\end{table*}

\paragraph{\conditional obviously won't generalize to unseen classes. Why aren't there context-conditioned baselines? }
This is exactly the argument we are trying to make. Any model trained on conditioning variables will only work for distributions representable within that conditioning variable. Therefore, these models rely heavily on diverse training in order to generalize to related but unseen context variables. 
This limitation is a major motivation for our work and why we seek a context-conditioned model. 
Note that we aren't suggesting there is no place for conditional models. They are a reasonable approach in many settings. However, our results show that a context-conditioned model can perform comparatively, and excels in some areas where conditional models fail. 
In terms of other context-conditioned baselines, we use \cg, \dg, and \finetune as representative approaches. 
One other obvious choice is to directly condition a transformer on the context, which consists of 10's of images. However, under a standard image processor such as Pan and Scan \citep{gemmateam2025gemma3technicalreport}, each image is converted into hundreds of tokens, leading to a joint context space of tens of thousands of tokens, not counting the generated image itself. 
While this is mathematically reasonable, the quadratic cost of attention implies it far exceeds all other approaches in expense. Thus, it is beyond the scope of this work.

\paragraph{Why are the basis functions independent?}
As argued in prior work \citep{fe}, the basis functions are almost surely independent upon initialization. The argument is based on sampling $k$ vectors from an infinite dimensional vector space. During training, dependent basis functions have higher loss, and hence the basis functions to not converge towards dependence. Thus, we find that the basis is independent in practice. 

\paragraph{Is it an issue that Algorithm \ref{alg: stfe} uses the same data to generate $x_t^{\iota, i}$ as to compute the empirical velocity estimate? }

We run an ablation where we prevent the $x_0^{\iota, i}$ that was used to generate $x_t^{\iota, i}$ from affecting the empirical estimate. To do so, we set its log probability to negative infinity before normalizing,  guaranteeing its probability is set to 0. We call this variant \loo. We find that it has slightly worse recall, but otherwise similar performance as \dynamic. So, this detail is not of much practical relevance. See Table \ref{tab:mnist_quant_table_loo}.

\section{Algorithms} \label{app:algs}

In this section, we provide the algorithm for each of our methods. Note that they largely follow the algorithm in \citep{fe}. The main difference between them is the sampling procedure. Additionally, function encoders typically use different data sets for calibrating the coefficients and computing the loss. However, this requires forward passing both datasets, which is expensive for image generation. Instead we use the same data for both calibration and training loss, which improves computational efficiency. 

\begin{algorithm}[h]
\caption{\static}
\label{alg: cfe}
\begin{algorithmic}
\STATE \textbf{given} datasets $\{ \mathcal{D}^\iota\}_{\iota \in \mathcal{T} }$, learning rate $\alpha$
\STATE Initialize basis $\lbrace g_{1}, \ldots, g_{k} \rbrace$ with parameters $\theta$
\WHILE{not converged}
\STATE   \algcomment{Initialize loss for gradient accumulation.}
\STATE $L \gets 0$
\STATE   \algcomment{For every distribution in the training set.}
\FORALL{$\iota \in \mathcal{T}$}
    \STATE   \algcomment{Define the samples.}
   \STATE $\{x_1^{\iota,i}\}_{i=1}^m \gets \mathcal{D}^\iota$
   \STATE $\{x_0^{\iota,i}\}_{i=1}^m \sim \mathcal{N}(0, I)$
   \STATE $\{t^i\}_{i=1}^m \sim \text{Unif}([0,1])$
   \STATE   \algcomment{Standard $X_t$ definition.}
   \STATE $x_t^{\iota,i} \gets (1-t^i) x_0^{\iota,i} + t^i x_1^{\iota,i}   $
   \STATE   \algcomment{Compute coefficients. The inner product is approximated using the samples.}
    \STATE $c^\iota \gets 
    \left ( \begin{bmatrix}
    \langle g^1, g^1 \rangle_{p_{t,X_t^\iota}} & \hdots & \langle g^1, g^k \rangle_{p_{t,X_t^\iota}} \\
    \vdots & \ddots & \vdots \\
    \langle g^k, g^1 \rangle_{p_{t,X_t^\iota}} & \hdots & \langle g^k, g^k \rangle_{p_{t,X_t^\iota}}
    \end{bmatrix} + \lambda I \right ) ^{-1}
    \begin{bmatrix}
    \langle v^\iota, g^1 \rangle_{p_{t,X_t^\iota}} \\
    \vdots \\
    \langle v^\iota, g^k \rangle_{p_{t,X_t^\iota}}
    \end{bmatrix}
    $
    \STATE \algcomment{Define the approximated velocity field.}
    \STATE $\Hat{v}^\iota(\cdot, \cdot) := \sum_{i=1}^k c^{\iota, i} g^i(\cdot, \cdot)$

    \STATE \algcomment{Gradient accumulation. The norm is again approximated using the samples.}
    \STATE $L \gets L + \lVert v^\iota - \Hat{v}^\iota \rVert^{2}_{p_{t,X_t^\iota}}$  
\ENDFOR

\algcomment{Gradient descent.}
\STATE $\theta \gets \theta - \alpha \nabla_\theta (L)$ 
\ENDWHILE
\RETURN $\lbrace g_{1}, \ldots, g_{k} \rbrace$
\end{algorithmic}
\end{algorithm}

\begin{algorithm}[h]
\caption{\temporal}
\label{alg: tfe}
\begin{algorithmic}
\STATE \textbf{given} datasets $\{ \mathcal{D}^\iota\}_{\iota \in \mathcal{T} }$, learning rate $\alpha$
\STATE Initialize basis $\lbrace g_{1}, \ldots, g_{k} \rbrace$ with parameters $\theta$
\WHILE{not converged}
\STATE   \algcomment{Initialize loss for gradient accumulation.}
\STATE $L \gets 0$
\STATE   \algcomment{For every distribution in the training set.}
\FORALL{$\iota \in \mathcal{T}$}
    \STATE \algcomment{Important: Sample a single time t}
    \STATE $t \sim \text{Unif}([0,1])$ 
    \STATE  \algcomment{Define the samples.}
   \STATE $\{x_1^{\iota,i}\}_{i=1}^m \gets \mathcal{D}^\iota$
   \STATE $\{x_0^{\iota,i}\}_{i=1}^m \sim \mathcal{N}(0, I)$
   
   \STATE   \algcomment{Standard $X_t$ definition.}
   \STATE $x_t^{\iota,i} \gets (1-t) x_0^{\iota,i} + t x_1^{\iota,i}$   
   \STATE   \algcomment{Compute coefficients. The inner product is approximated using the samples.}
   \STATE $c^\iota(t) \gets \left ( \begin{bmatrix}
\langle g^1_t, g^1_t \rangle_{p_{X_t^\iota|t}} & \hdots & \langle g^1_t, g^k_t \rangle_{p_{X_t^\iota|t}} \\
\vdots & \ddots & \vdots \\
\langle g^k_t, g^1_t \rangle_{p_{X_t^\iota|t}} & \hdots & \langle g^k_t, g^k_t \rangle_{p_{X_t^\iota|t}} \\
\end{bmatrix} + \lambda I \right )^{-1}
\begin{bmatrix}
\langle v^\iota, g^1_t \rangle_{p_{X_t^\iota|t}} \\
\vdots \\
\langle v^\iota, g^k_t \rangle_{p_{X_t^\iota|t}} \\
\end{bmatrix}$
    \STATE   \algcomment{Define the approximated velocity field.}
    \STATE $\Hat{v}_t^\iota(\cdot) := \sum_{i=1}^k c^{\iota, i}(t) g^i_t(\cdot)$

    \STATE \algcomment{Gradient accumulation. The norm is again approximated using the samples.}
    \STATE $L \gets L + \lVert v^\iota_t - \Hat{v}_t^\iota \rVert^{2}_{p_{X_t^\iota|t}}$  
\ENDFOR

\algcomment{Gradient descent.}
\STATE $\theta \gets \theta - \alpha \nabla_\theta (L)$ 
\ENDWHILE
\RETURN $\lbrace g_{1}, \ldots, g_{k} \rbrace$
\end{algorithmic}
\end{algorithm}

\begin{algorithm}[h]
\caption{\dynamic }
\label{alg: stfe}
\begin{algorithmic}
\STATE \textbf{given} datasets $\{ \mathcal{D}^\iota\}_{\iota \in \mathcal{T} }$, learning rate $\alpha$
\STATE Initialize basis $\lbrace g_{1}, \ldots, g_{k} \rbrace$ with parameters $\theta$
\WHILE{not converged}
\STATE   \algcomment{Initialize loss for gradient accumulation.}
\STATE $L \gets 0$
\STATE   \algcomment{For every distribution in the training set.}
\FORALL{$\iota \in \mathcal{T}$}
    \STATE   \algcomment{Define the samples.}
   \STATE $\{x_1^{\iota,i}\}_{i=1}^m \gets \mathcal{D}^\iota$
   \STATE $\{x_0^{\iota,i}\}_{i=1}^m \sim \mathcal{N}(0, I)$
   \STATE $\{t^i\}_{i=1}^m \sim \text{Unif}([0,1])$
   \STATE   \algcomment{Standard $X_t$ definition.}
   \STATE$ x_t^{\iota,i} \gets (1-t^i) x_0^{\iota,i} + t^i x_1^{\iota,i}   $
   \STATE \algcomment{Important: In the following for loop, we compute the noise samples from $x_t$ and $x_1$.}
   \STATE \algcomment{Therefore we delete the current noise samples to make this explicit.}
   \STATE Delete $\{x_0^{\iota,i}\}_{i=1}^m$
   \STATE \algcomment{For every state-time.}
   \FORALL{$(x_t,t) \in \{(x_t^{\iota,j}, t^j)\}_{j=1}^m $}
    \STATE \algcomment{Compute $x_0^*$ as in Theorem \ref{thm: sampling}}
    \STATE \algcomment{For each $x_t$, compute $x_0^{\iota,i}$ for all $i$, meaning we have a set $\{x_0^{\iota,i}\}_{i=1}^m$}
    \STATE  $x_0^{\iota,i} \gets \frac{x_t - t x_1^{\iota, i}}{1-t}   $
    \STATE   \algcomment{Compute coefficients. The inner product is approximated using the samples.}
    \STATE \algcomment{See Appendix \ref{app: sampling}.}
       \STATE $c^\iota(x_t,t) \gets \left ( \begin{bmatrix}
    \langle g^1, g^1 \rangle_{p_{t,X_t^\iota|t,X_t}} & \hdots & \langle g^1, g^k \rangle_{p_{t,X_t^\iota|t,X_t}} \\
    \vdots & \ddots & \vdots \\
    \langle g^k, g^1 \rangle_{p_{t,X_t^\iota|t,X_t}} & \hdots & \langle g^k, g^k \rangle_{p_{t,X_t^\iota|t,X_t}} \\
    \end{bmatrix} + \lambda I \right )^{-1}
    \begin{bmatrix}
    \langle v^\iota, g^1 \rangle_{p_{t,X_t^\iota|t,X_t}} \\
    \vdots \\
    \langle v^\iota, g^k \rangle_{p_{t,X_t^\iota|t,X_t}} \\
    \end{bmatrix}$
        \STATE   \algcomment{Define the approximated velocity field.}
        \STATE $\Hat{v}^\iota(x_t,t) := \sum_{i=1}^k c^{\iota, i}(x_t,t) g^i(x_t,t)$
    
        \STATE \algcomment{Gradient accumulation. The norm is again approximated using the samples.}
        \STATE $L \gets L + \lVert v^\iota - \Hat{v}^\iota \rVert^{2}_{p_{t,X_t^\iota|t,X_t}}$  
    
   \ENDFOR

\ENDFOR

\algcomment{Gradient descent.}
\STATE $\theta \gets \theta - \alpha \nabla_\theta (L)$ 
\ENDWHILE
\RETURN $\lbrace g_{1}, \ldots, g_{k} \rbrace$
\end{algorithmic}
\end{algorithm}

\end{document}